\documentclass[letterpaper]{article} 
\usepackage{aaai23}  
\usepackage{times}  
\usepackage{helvet}  
\usepackage{courier}  
\usepackage[hyphens]{url}  
\usepackage{graphicx} 
\usepackage{caption} 
\usepackage{algorithm}
\usepackage{amssymb,amsthm}
\usepackage{amsmath}
\usepackage{algorithmic}
\usepackage{comment}
\usepackage{booktabs}
\usepackage{mathtools}
\usepackage{tabularx}
\usepackage{multirow}
\usepackage[table,xcdraw]{xcolor}
\newtheorem{theorem}{Theorem}
\usepackage{subfigure}

\usepackage{newfloat}
\usepackage{listings}
\DeclareCaptionStyle{ruled}{labelfont=normalfont,labelsep=colon,strut=off} 
\floatstyle{ruled}
\newfloat{listing}{tb}{lst}{}
\floatname{listing}{Listing}
\title{
Beyond Graph Convolutional Network: An Interpretable Regularizer-centered Optimization Framework
}
\author {
    Shiping Wang\textsuperscript{\rm 1,\rm 2},
    Zhihao Wu\textsuperscript{\rm 1,\rm 2},
    Yuhong Chen\textsuperscript{\rm 1,\rm 2},
    Yong Chen\textsuperscript{\rm 3}\thanks{Corresponding author}
}
\affiliations {
    \textsuperscript{\rm 1} College of Computer and Data Science, Fuzhou University\\
    \textsuperscript{\rm 2} Fujian Provincial Key Laboratory of Network Computing and Intelligent Information
Processing, Fuzhou University\\
    \textsuperscript{\rm 3} School of Computer Science, Beijing University of Posts and Telecommunications\\
    shipingwangphd@163.com, zhihaowu1999@gmail.com, yhchen2320@163.com, alphawolf.chen@gmail.com.
}

\begin{document}
\maketitle

\begin{abstract}

Graph convolutional networks (GCNs) have been attracting widespread attentions due to their encouraging performance and powerful generalizations. 
However, few work provide a general view to interpret various GCNs and guide GCNs' designs. 
In this paper, by revisiting the original GCN, we induce an interpretable regularizer-centerd optimization framework, in which by building appropriate regularizers we can interpret most GCNs, such as APPNP, JKNet, DAGNN, and GNN-LF/HF.
Further, under the proposed framework, we devise a dual-regularizer graph convolutional network (dubbed tsGCN) to capture topological and semantic structures from graph data. 
Since the derived learning rule for tsGCN contains an inverse of a large matrix and thus is time-consuming, we leverage the Woodbury matrix identity and low-rank approximation tricks to successfully decrease the high computational complexity of computing infinite-order graph convolutions. 
Extensive experiments on eight public datasets demonstrate that tsGCN achieves superior performance against quite a few state-of-the-art competitors w.r.t. classification tasks.
\end{abstract}

\section{Introduction}
Owing to the powerful ability to aggregate neighborhood information, Graph Convolutional Network (GCN) has been successfully applied to diverse domains, such as computer vision \cite{ChenWWG19, NieZLGS20, WangCFP22}, recommender systems \cite{XuLHLX019, ChenHWL22}, privacy preserving \cite{HuCVB22}, and traffic forecasting \cite{YuYZ18, ChenCXCGF20}. 
Rooted in a series of theoretical foundations, GCN extends convolution operations to the non-Euclidean spaces and effectively propagates label signals, and therefore its variants have been extensively employed for a variety of graph-related tasks, including classification \cite{ZhangPCU19, yang2022new}, clustering \cite{FanWSLLW20, ZhuK21} and link prediction \cite{ChenYS0GM20, Halliwell22}. 
In a nutshell, GCN generates the graph embedding with the well-established graph convolutional layers gathering semantics from neighbors according to the network topology, which are revealed to be the most critical component.

Although GCN has behaved well in many machine learning tasks, lots of studies have pointed out its certain drawbacks and made efforts for further improvements. 
Bo et al. \cite{BoWSS21} indicated that the propagation mechanism could be considered as a special form of low-pass filter, and presented a GCN with an adaptive frequency. 
Zhang et al. \cite{9568649} argued that most GCN-based methods ignored the global information and proposed SHNE, which leveraged the structure and feature similarity to capture latent semantics.
Wang et al. \cite{wang2020gcn} revealed that the original GCN aggregated information from node neighbors inadequately, and then developed a multi-channel GCN by utilizing feature-based semantic graph.
In spite of the performance boosts of these GCN variants, they didn't establish a generalized framework, i.e., these approaches understood and enhanced GCN from certain and non-generalizable perspectives, thereby they are exceedingly difficult to be further developed, and with limited interpretability.

Consequently, it is expected to construct a unified framework for various GCNs with better interpretability; however, it is a pity that this kind of work is still in shortage. 
Zhao et al. \cite{zhao2020connecting} linked GCN and Graph-regularized PCA (GPCA), and then proposed a multi-layer network by stacking the GPCA layers. 
Zhu et al. \cite{ZhuWang2021Interpreting} attempted to interpret existing GCN-based methods with a unified optimization framework, under which they devised an adjustable graph filter for a new GCN variant.
Yang et al. \cite{YangLWZGWZHW21} designed a family of graph convolutional layers inspired by the updating rules of two typical iterative algorithms.
Although these efforts have contributed to better understanding of GCNs, they only explained GCNs in partial aspects, promoting the expectation of a more comprehensive analysis of GCNs.

To tackle the aforementioned issues, this paper induces an interpretable regularizer-centered optimization framework, which provides a novel perspective to digest various GCNs, i.e., this framework captures the common essential properties of existing state-of-the-art GCN variants and could defines them just by devising different regularizers. 
Moreover, in light of the analyses on current representative GCNs, we find that most of the existing approaches only consider capturing the topological regularization, while the feature-based semantic structure is underutilized, and hence this motivates us to design a dual-regularizer graph convolutional network (called tsGCN) within the regularizer-centered optimization framework for the fullest explorations of both structures and semantics from graph data.
Due to the high computational complexity of performing infinite-order graph convolutions, the unified framework provides a straightforward way employing truncated polynomials to approximate the graph Laplacian, similar to the truncated Chebyshev polynomials by vanilla GCN, restricting the message passing of a single graph convolution to the first-order neighborhood.

The main contributions of this paper can be summarized as the following three aspects:
\begin{itemize}
    \item Propose a regularizer-centered constrained optimization framework, which interprets various existing GCNs with specific regularizers.
    
    \item Establish a new dual-regularizer graph convolutional network (tsGCN), which exploits topological and semantic structures of the given data; and develop an efficient algorithm to reduce the computational complexity of solving infinite-order graph convolutions.
    
    \item Conduct a series of experiments to show that tsGCN performs much better than many SOTA GCNs, and also consumes much less time than the newly GNN-HF/LF. 
\end{itemize}

\section{Related Work}
\subsection{Graph Convolutional Networks}
The original GCN was first introduced by Kipf et al. \cite{KipfWelling2017SemiSupervised}, who generalized the convolution operations from the Euclidean domain to the non-Euclidean domain.
SGC \cite{WuJrZhang2019Simplifying} assumed that the nonlinear transform of GCN was not that significant, and then devised a simplified GCN by removing the nonlinear activation functions and collapsing the weight matrices.
PPNP \cite{KlicperaBojchevski2019Predict} employed the relationship between PageRank and GCN for the improvement on the propagation mechanism of GCN, and an iterative version called APPNP was further proposed to reduce the high computational complexity.
Attempting to adaptively learn the influence radii for each node and task, JKNet \cite{XuLiTian2018Representation} combined various aggregations at the last layer and was able to learn representations of different orders for graph substructures.
GNN-LF and GNN-HF \cite{ZhuWang2021Interpreting} considered the low-pass and the high-pass filter as the convolution kernels to improve GCN's expressive power, respectively.
AdaGCN \cite{SunZL21} leveraged Adaboost strategy for the enhancement of GCN, allowing information to be shared between layers.
To sum up, a main characteristic of these methods is exploring GCN from the perspectives of redesigning information aggregation strategies or modifying graph convolutions, and few work try to construct a unified framework to interpret various GCNs and reveal the underlying common principles. 

\subsection{Further Insights on GCNs}
Quite a few studies have been devoted to explore the mechanisms of GCN for deeper insights.
Li et al. \cite{LiHW18} indicated that the convolutional operation of GCN was a special form of Laplacian smoothing, attributed to which GCN suffered from the so-called over-smoothing problem.
Specifically, the performance of GCN will decrease as the number of layers increases, which has been validated by many other studies.
However, Liu et al. \cite{LiuGaoJi2020Towards} held a different opinion that the entanglement of two steps in GCN damages the performance of the deep GCN, where the two steps were explained as propagation and transformation. Based on this view, they decoupled the two operations and further presented a deeper GCN.
Zhu et al. \cite{ZhuWang2021Interpreting} also decomposed the convolution operation of GCN into two separate stages, called aggregation and transformation, and focused on the aggregation process, formulating an optimization objective to interpret it. 
Yang et al. \cite{0002KCJ0G19} explored network topology refinement, leveraging a topology optimization process for the explanation.
Oono et al. \cite{OonoS20} analyzed the forward propagation of GCN and interpreted it with a specific dynamical system, allowing GCN to be related to the underlying topological structures.
Overall, these studies have contributed to the interpretability of GCNs, and also let researchers better understand GCNs. 
In this paper, we build a unified optimization framework from a novel view of graph regularizers to interpret and understand GCNs.

\section{Mathematical Notations}

For the convenience of formal descriptions, derivations, and analyses, necessary notations are narrated as below.
A graph is denoted as $\mathcal{G} = (\mathcal{V}, \mathcal{E}, \mathbf{A})$, where $\mathcal{V}$ marks the vertex set with $|\mathcal{V}|=N$ ($N$ is the total number of nodes in graph $\mathcal{G}$), $\mathcal{E}$ marks the edge set, and $\mathbf{A}=[\mathbf{A}_{ij}]_{N\times N}$ marks an affinity matrix of which $\mathbf{A}_{ij}$ measures the similarity between the $i$-th and the $j$-th node.
In addition, $\mathbf{D}=[\mathbf{D}_{ij}]_{N\times N}$ represents the degree matrix of $\mathcal{G}$ with $\mathbf{D}_{ii}=\sum_{j=1}^{N}\mathbf{A}_{ij}$, and then the normalized symmetrical graph Laplacian of $\mathcal{G}$ is computed as $\widetilde{\mathbf{L}} = \mathbf{I} - \widetilde{\mathbf{A}}$ with $\widetilde{\mathbf{A}} = \mathbf{D}^{-\frac{1}{2}}\mathbf{A}\mathbf{D}^{-\frac{1}{2}}$.

\section{Revisiting Graph Convolutional Network}

\begin{table*}[!htbp]
    \centering
    \resizebox{\linewidth}{!}{
    \begin{tabular}{|m{2cm}<{\centering}|m{7.8cm}<{\centering}|m{6.6cm}<{\centering}|m{4.7cm}|}
    \toprule		
     \multicolumn{1}{|c|}{Methods} &  \multicolumn{1}{c|}{Propagation Rules} & \multicolumn{1}{c|}{Regularizer $\mathcal{L}(\mathbf{H}^{(l)};\mathcal{G})$} & \multicolumn{1}{c|}{Projective Set} \\
    \midrule
    GCN   & $\mathbf{H}^{(l)} = \sigma\left(\widehat{\mathbf{A}}\mathbf{H}^{(l-1)}\mathbf{\Theta}^{(l)}\right)$  & $ Tr\left({\mathbf{H}^{(l)}}^{\top}\widetilde{\mathbf{L}}\mathbf{H}^{(l)}\right)$ & $\left\{\begin{matrix}\mathcal{S}^{(l)} = \mathcal{S}_{+}, l\in[L$$-$$1],\\ \mathcal{S}^{(L)} =\mathcal{S}_{simplex}\end{matrix}\right.$\\
    SGC  & $\mathbf{H}^{(l)} = \sigma\left({\widehat{\mathbf{A}}}\mathbf{H}^{(l-1)}\mathbf{\Theta}^{(l)}\right)$ &  $Tr\left({\mathbf{H}^{(l)}}^{\top}\widetilde{\mathbf{L}}\mathbf{H}^{(l)}\right)$ & $\left\{\begin{matrix}\mathcal{S}^{(l)} = \mathcal{S}, l\in[L$$-$$1],\\ \mathcal{S}^{(L)} =\mathcal{S}_{simplex}\end{matrix}\right.$\\
    APPNP & $\mathbf{H}^{(l)} = \sigma\left((1-\alpha)\widehat{\mathbf{A}}\mathbf{H}^{(l-1)}+\alpha\mathbf{H}^{(0)}\right)$ & $ Tr\left(\frac{1}{1-\alpha}{\mathbf{H}^{(l)}}^{\top}{\widehat{\mathbf{A}}}^{-1}(\mathbf{H}^{(l)}-2\alpha\mathbf{H}^{(0)})\right)$ & $\left\{\begin{matrix}\mathcal{S}^{(l)} = \mathcal{S}, l\in[L$$-$$1],\\ \mathcal{S}^{(L)} =\mathcal{S}_{simplex}\end{matrix}\right.$\\
    JKNet & $\mathbf{H}^{(l)} = \sigma\left(\sum_{k=1}^{K}\alpha_{k}{\widehat{\mathbf{A}}}^{k}\mathbf{H}^{(l-1)}\mathbf{\Theta}^{(l)}\right)$ & $Tr\left({\mathbf{H}^{(l)}}^{\top}{\widehat{\mathbf{A}}}^{-1}(\mathbf{I} + \beta\widetilde{\mathbf{L}})\mathbf{H}^{(l)}\right)$   & $\left\{\begin{matrix}\mathcal{S}^{(l)} = \mathcal{S}, l\in[L$$-$$1],\\ \mathcal{S}^{(L)} =\mathcal{S}_{simplex}\end{matrix}\right.$\\
    DAGNN  & $\mathbf{H}^{(l)} =  \sigma\left(\sum_{k=0}^{K}\alpha_{k}{\widehat{\mathbf{A}}}^{k}\mathbf{H}^{(l-1)}\right)$ & $Tr\left({\mathbf{H}^{(l)}}^{\top}(\mathbf{I} + \beta\widetilde{\mathbf{L}})\mathbf{H}^{(l)}\right)$   & $\left\{\begin{matrix}\mathcal{S}^{(l)} = \mathcal{S}, l\in[L$$-$$1],\\ \mathcal{S}^{(L)} =\mathcal{S}_{simplex}\end{matrix}\right.$ \\
    GNN-HF &  $\mathbf{H}^{(l)} = \sigma\left((\mathbf{I} + \alpha\widehat{\mathbf{L}})^{-1}(\mathbf{I} + \beta\widehat{\mathbf{L}})\mathbf{H}^{(l-1)}\mathbf{\Theta}^{(l)}\right)$ & $Tr\left({\mathbf{H}^{(l)}}^{\top}(\mathbf{I} + \beta\widehat{\mathbf{L}})^{-1}(\mathbf{I} + \alpha\widehat{\mathbf{L}})\mathbf{H}^{(l)}\right)$ & $\left\{\begin{matrix}\mathcal{S}^{(l)} = \mathcal{S}_{+}, l\in[L$$-$$1],\\ \mathcal{S}^{(L)} =\mathcal{S}_{simplex}.\end{matrix}\right.$\\
    GNN-LF &  $\mathbf{H}^{(l)} = \sigma\left((\mathbf{I} + \alpha\widehat{\mathbf{A}})^{-1}(\mathbf{I} + \beta\widehat{\mathbf{A}})\mathbf{H}^{(l-1)}\mathbf{\Theta}^{(l)}\right)$ & $Tr\left({\mathbf{H}^{(l)}}^{\top}(\mathbf{I} + \beta\widehat{\mathbf{A}})^{-1}(\mathbf{I} + \alpha\widehat{\mathbf{A}})\mathbf{H}^{(l)} \right)$ & $\left\{\begin{matrix}\mathcal{S}^{(l)} = \mathcal{S}_{+}, l\in[L$$-$$1],\\ \mathcal{S}^{(L)} =\mathcal{S}_{simplex}\end{matrix}\right.$\\
    \midrule
    tsGCN & $\mathbf{H}^{(l)} = \sigma\left((\mathbf{I} +  \alpha\widetilde{\mathbf{L}}_{\mathcal{G}} + \beta\widetilde{\mathbf{L}}_{\mathcal{X}})^{-1}\mathbf{H}^{(l-1)}\mathbf{\Theta}^{(l)}\right)$ & $Tr\left({\mathbf{H}^{(l)}}^{\top}(\mathbf{I} + \alpha\widetilde{\mathbf{L}}_{\mathcal{G}} + \beta\widetilde{\mathbf{L}}_{\mathcal{X}})\mathbf{H}^{(l)}\right)$ &  $\left\{\begin{matrix}\mathcal{S}^{(l)} = \mathcal{S}_{+}, l\in[L$$-$$1],\\ \mathcal{S}^{(L)} =\mathcal{S}_{simplex}\end{matrix}\right.$\\
    \bottomrule
    \end{tabular}}
    \caption{Different regularizers can derive different GCN variants under the regularizer-centered optimization framework.}
    \label{ConnectionRegularizerGCN}
\end{table*}

For a graph $\mathcal{G} = (\mathcal{V}, \mathcal{E}, \mathbf{A})$, the svd of its graph Laplacian is $\mathbf{L} = \mathbf{U}\mathbf{\Lambda}\mathbf{U}^{\top}$, where $\mathbf{U}\in\mathbb{R}^{N\times N}$ is comprised of orthonormal eigenvectors and $\mathbf{\Lambda} = \mbox{diag}(\lambda_{1}, \cdots, \lambda_{N})$ is a diagonal matrix with $\lambda_i$ denoting the $i$-th eigenvalue and $\lambda_i \ge \lambda_j$ ($i=1,\cdots,N$).
Essentially, this decomposition induces a Fourier transform on the graph domain, where eigenvectors correspond to Fourier components and eigenvalues represent frequencies of the graph.
For an input signal $\mathbf{x}\in\mathbb{R}^{N}$ defined on the graph $\mathcal{G}$, the corresponding graph Fourier transform of $\mathbf{x}$ is $\widehat{\mathbf{x}} = \mathbf{U}^{\top}\mathbf{x}$, and its inverse transform is derived as $\mathbf{x} = \mathbf{U}\widehat{\mathbf{x}}$.
Consequently, the graph convolution between the signal $\mathbf{x}$ and the filter $\mathbf{g}\in\mathbb{R}^{N}$ is
\begin{eqnarray}
\mathbf{g} * \mathbf{x} = \mathbf{U}(\widehat{\mathbf{g}}\odot\widehat{\mathbf{x}}) = \mathbf{U}((\mathbf{U}^{\top}\mathbf{g})\odot(\mathbf{U}^{\top}{\mathbf{x}})),
\end{eqnarray}
where $\odot$ is the Hadamard product between two vectors.
Particularly, denoting $\mathbf{g}_{\mathbf{\Theta}}=\mbox{diag}(\mathbf{\Theta}) := \mathbf{U}^{\top}\mathbf{g}$ parameterized by $\mathbf{\Theta}\in\mathbb{R}^{N}$, the graph convolution between $\mathbf{x}$ and $\mathbf{g}$ can be rewritten as
\begin{eqnarray}
\mathbf{g} * \mathbf{x} = \mathbf{U}(\widehat{\mathbf{g}}\odot\widehat{\mathbf{x}}) = \mathbf{U}\mathbf{g}_{\mathbf{\Theta}}\mathbf{U}^{\top}{\mathbf{x}},
\end{eqnarray}
where $\mathbf{\Theta}$ is regarded as the filter coefficients to be optimized.
Especially, $\mathbf{\Theta}$ is assumed to be the polynomials of the spectrums of the graph Laplacian \cite{HammondVander20111Wavelets}, i.e.,
\begin{eqnarray}
\mathbf{\Theta} = \mathbf{\Theta}(\mathbf{\Lambda}) = \sum_{i=1}^{K}\mathbf{\Theta}_{i}\mathbf{\Lambda}^{i},
\end{eqnarray}
where $K$ is the order of Chebyshev polynomials.
By fixing $K=2$, the graph convolutional network (GCN) \cite{KipfWelling2017SemiSupervised} takes an effective form
\begin{eqnarray}
\mathbf{g} * \mathbf{x} = \theta(\mathbf{I} + \mathbf{L})\mathbf{x},
\end{eqnarray}
where $\mathbf{\Theta}=[\theta]$ is a parameter to be optimized.
When extending single channel signal $\mathbf{x}$ and filter $\theta$ to multi-channel $\mathbf{H}^{(l)}\in\mathbb{R}^{N\times d_{l}}$ and $\mathbf{\Theta}^{(l)}\in\mathbb{R}^{d_{l}\times f_{l}}$, the GCN is converted to
\begin{eqnarray}\label{GCNupdatingRule}
\mathbf{H}^{(l)} = \sigma(\widehat{\mathbf{A}}\mathbf{H}^{(l-1)}\mathbf{\Theta}^{(l)}),
\end{eqnarray}
where $\widehat{\mathbf{A}}$ is a normalized version of $\mathbf{I}+\widetilde{\mathbf{A}}$, $\sigma(\cdot)$ is an activation function, and $\mathbf{H}^{(l)}\in\mathbb{R}^{N\times d_{l}}$ is the output of the $l$-th layer with $\mathbf{H}^{(0)} = \mathbf{X}$ being the input feature matrix.

\section{An Interpretable Regularizer-centered Optimization Framework for GCNs}

Given the input $\mathbf{H}^{(l-1)}$ of the $(l)$-th layer, GCN can compute the output $\mathbf{H}^{(l)}$ by minimizing
\begin{align}\label{manifoldEmbedding}
\begin{split}
\mathcal{L} = -\text{Tr}({\mathbf{H}^{(l)}}^{\top}\mathbf{H}^{(l-1)}\mathbf{\Theta}^{(l)}) + \frac{1}{2}\text{Tr}({\mathbf{H}^{(l)}}^{\top}\widetilde{\mathbf{L}}\mathbf{H}^{(l)}) \\
\end{split}
\end{align}
\begin{displaymath}
\mbox{ s.t. } \mathbf{H}^{(l)} \geq \mathbf{0},
\end{displaymath}
where $\frac{1}{2}\text{Tr}({\mathbf{H}^{(l)}}^{\top}\widetilde{\mathbf{L}}\mathbf{H}^{(l)}) = \frac{1}{4}\sum_{j=1}^{N}\sum_{i=1}^{N}\mathbf{A}_{ij}||\frac{\mathbf{h}^{(l)}_{i}}{\sqrt{\mathbf{D}_{ii}}} -\frac{\mathbf{h}^{(l)}_{j}}{\sqrt{\mathbf{D}_{jj}}}||_{2}^{2}$ with $\mathbf{H}^{(l)} = [\mathbf{h}_{1}^{(l)}; \cdots; \mathbf{h}_{N}^{(l)}]$; it is a normalized regularizer to preserve the pairwise similarity of any two nodes in the given graph. 
Besides, the $-\text{Tr}({\mathbf{H}^{(l)}}^{\top}\mathbf{H}^{(l-1)}\mathbf{\Theta}^{(l)})$ is actually a fitting loss term bewteen $\mathbf{H}^{(l)}$ and $\mathbf{H}^{(l-1)}\mathbf{\Theta}^{(l)}$, i.e., $||\mathbf{H}^{(l)}-\mathbf{H}^{(l-1)}\mathbf{\Theta}^{(l)}||_{F}^{2}$ with $\mathbf{H}^{(l-1)}$ and $\mathbf{\Theta}^{(l)}$ fixed when optimizing $\mathbf{H}^{(l)}$. Note that the square term $||\mathbf{H}^{(l)}||_{F}^{2}$ is a L2-regularized smoother, which can be ignored or absorbed in the second graph regularizer $\text{Tr}({\mathbf{H}^{(l)}}^{\top}\widetilde{\mathbf{L}}\mathbf{H}^{(l)})$.

Taking derivative of $\mathcal{L}$ with respect to $\mathbf{H}^{(l)}$ and setting it to zero, we obtain $\mathbf{H}^{(l_{+})}$ as
\begin{align}\label{ClosedFormME}
\mathbf{H}^{(l_{+})} = (\mathbf{I}-\widetilde{\mathbf{A}})^{-1} \mathbf{H}^{(l-1)}{\mathbf{\Theta}^{(l)}}; 
\end{align}
and then there yields
\begin{equation}\label{ClosedFormME_Final}
   \mathbf{H}^{(l)}=\sigma\left( \mathbf{H}^{(l_{+})} \right),
\end{equation}
when the nonnegative constraints $\mathbf{H}^{(l)} \geq \mathbf{0}$ are further considered.
Notice that $\sigma(\cdot)$ is the $\mathrm{ReLU}(\cdot)$ activation function. 
Here, if the matix inverse $(\mathbf{I}-\tilde{\mathbf{A}})^{-1}  = \sum_{i=0}^{\infty}\tilde{\mathbf{A}}^{i}$ is approximated by the first-order expansion, i.e., $(\mathbf{I}-\tilde{\mathbf{A}})^{-1}  \approx \mathbf{I} + \tilde{\mathbf{A}}$, then Eq.~\eqref{ClosedFormME_Final} will lead to the updating rule \eqref{GCNupdatingRule} of GCN. 

Usually, the activation functions in GCN are $\mathrm{ReLU}(\cdot)$ and $\mathrm{Softmax}(\cdot)$, which could be converted to different projection optimizations.
Concretely, the $\mathrm{ReLU}(\cdot)$ activation function is equivalent to project a point $\mathbf{x}$ onto the non-negative plane $\mathcal{S}_{+} = \{\mathbf{s}\in\mathbb{R}^{d} | \mathbf{s} \geq \mathbf{0}\}$, i.e., 
\begin{align}
\mathrm{ReLU}(\mathbf{x}) = \mathop{\arg\min}_{\mathbf{y}\in\mathcal{S}_{+}} -\mathbf{x}^{\top}\mathbf{y} + \frac{1}{2}||\mathbf{y}||_{2}^{2}. 
\end{align}
By the way, we denote $\mathcal{S} = \{\mathbf{s}\in\mathbb{R}^{d}\}$, which corresponds to an identity activation function.
In terms of the $\mathrm{Softmax}(\cdot)$ activation function, it can be regarded as projecting $\mathbf{x}$ onto the set $\mathcal{S}_{simplex} = \{\mathbf{s}\in\mathbb{R}^{d} | \mathbf{1}^{\top}\mathbf{s} = 1, \mathbf{s} \geq \mathbf{0}\}$, i.e., 
\begin{align}\label{projection_simplex}
\mathrm{Softmax}(\mathbf{x}) = \mathop{\arg\min}_{\mathbf{y}\in\mathcal{S}_{simplex}} -\mathbf{x}^{\top}\mathbf{y} + \mathbf{y}^{\top}\log(\mathbf{y}),
\end{align}
where $\mathbf{y}^{\top}\log(\mathbf{y}) = \sum_{i=1}^{d}\mathbf{y}_{i}\log(\mathbf{y}_{i})$ is the negative entropy of $\mathbf{y}$ \cite{Amos2019Differentiable}. 
In fact, with respect to other activation functions, they can also be equivalent to project a point onto some feasible set with some metric.

Up to present, we have actually utilized a constrained optimization problem to interpret GCN, including information propagations (i.e., Eq.~\eqref{ClosedFormME}) and the nonlinear activation functions (i.e., $\mathrm{ReLU}(\cdot)$ and $\mathrm{Softmax}(\cdot)$).

The above analyses can not only explain the vanilla GCN, but also stimulate a regularizer-centered optimization framework that can further unify various GCNs. By extending the optimization \eqref{manifoldEmbedding}, a more general framework is written as
\begin{align}\label{regularizer_framework}
\begin{split}
\mathcal{L} = -\text{Tr}({\mathbf{H}^{(l)}}^{\top}\mathbf{H}^{(l-1)}\mathbf{\Theta}^{(l)}) + \frac{1}{2}\mathcal{L}(\mathbf{H}^{(l)};\mathcal{G}) \\
\end{split}
\end{align}
\begin{displaymath}
\mbox{ s.t. } \mathbf{H}^{(l)} \in \{\mathcal{S}_{+}\; or \;\mathcal{S}\}, l\in[L-1], \mathbf{H}^{(L)} \in \mathcal{S}_{simplex}.
\end{displaymath}

Under this framework, different regularizers could derive different GCNs, for example,
\begin{itemize}
    \item If $\mathcal{L}(\mathbf{H}^{(l)};\mathcal{G})=Tr\left({\mathbf{H}^{(l)}}^{\top}(\mathbf{I} + \mu\widehat{\mathbf{L}})^{-1}(\mathbf{I} + \lambda\widehat{\mathbf{L}})\mathbf{H}^{(l)}\right)$ with $\lambda = \beta + \frac{1}{\alpha} - 1$, $\mu = \beta$, and $\widehat{\mathbf{L}}=\mathbf{I}-\widehat{\mathbf{A}}$, then it induces the updating rule $\mathbf{H}^{(l)} = \sigma\left((\mathbf{I} + \alpha\widehat{\mathbf{A}})^{-1}(\mathbf{I} + \beta\widehat{\mathbf{A}})\mathbf{H}^{(l-1)}\mathbf{\Theta}^{(l)}\right)$, which corresponds to GNN-HF \cite{ZhuWang2021Interpreting}.
    
    \item If $\mathcal{L}(\mathbf{H}^{(l)};\mathcal{G})=Tr\left({\mathbf{H}^{(l)}}^{\top}(\mathbf{I} + \mu\widehat{\mathbf{A}})^{-1}(\mathbf{I} + \lambda\widehat{\mathbf{A}})\mathbf{H}^{(l)}\right)$ with $\lambda = \frac{-\alpha\beta+2\alpha-1}{\alpha\beta-\alpha+1}$ and $\mu = \frac{1}{\beta}-1$, then it gives rise to the updating rule $\mathbf{H}^{(l)} = \sigma\left((\mathbf{I} + \alpha\widehat{\mathbf{A}})^{-1}(\mathbf{I} + \beta\widehat{\mathbf{A}})\mathbf{H}^{(l-1)}\mathbf{\Theta}^{(l)}\right)$, which corresponds to GNN-LF \cite{ZhuWang2021Interpreting}.
\end{itemize}
For more cases, their results are summarized in Table \ref{ConnectionRegularizerGCN}, and the derivation details can refer to those of the original GCN (from Eq.~(\ref{ClosedFormME}) to Eq.~(\ref{projection_simplex})) and the supplementary.

\textbf{Remarks.} The work \cite{ZhuWang2021Interpreting} is most similar to our work with the same research idea: they both want to propose a unified framework to interpret the current GCNs and guide the design of new GCN variants; however, they are realized in different ways. To be specific, (1) Zhu et al. \cite{ZhuWang2021Interpreting} develop an optimization framework to explain different GCNs' propagation processes; whereas we propose a constrained optimization framework not only to interpret various GCNs' propagation processes, but also explain the nonlinear activation layers; (2) \cite{ZhuWang2021Interpreting} unifies various GCNs via devising various fitting items which are essentially constructed by limited graph filters; while our work derives different GCNs through designing different regularizers.
To sum up, our work interprets the whole (not partial) GCNs with regularizer-centered constrained optimizations.

\begin{algorithm}[!tbp]
    \caption{
    Topological and Semantic Regularized GCN}
    \label{alg:Framework}
    \begin{algorithmic}[1]
    \REQUIRE Graph data {$\mathcal{G} = (\mathcal{V}, \mathcal{E}, \mathbf{A})$, labels $\mathbf{y}$,
    number of layers $L$, and hyperparameters $\{\alpha, \beta, r\}$.}
    \ENSURE {Predicted label set $\{y_{i}^{*}\}_{i=n+1}^{N}$.}
           
    \STATE {Initialize model parameters $\{\mathbf{H}^{(l)}, \mathbf{\Theta}^{(l)}\}_{l=1}^{L}$; }
    \STATE {Compute the joint graph Laplacian $\alpha\widetilde{\mathbf{L}}_{\mathcal{G}} + \beta\widetilde{\mathbf{L}}_{\mathcal{X}}$ and its low-rank factorization $\mathbf{W}\mathbf{V}^{\top}$;}
    \STATE {Substitute the matrix inverse $(\mathbf{I} +  \alpha\widetilde{\mathbf{L}}_{\mathcal{G}} + \beta\widetilde{\mathbf{L}}_{\mathcal{X}})^{-1}$ with $\mathbf{I} - \mathbf{W}(\mathbf{I} +  \mathbf{V}^{\top}\mathbf{W})^{-1}\mathbf{V}^{\top}$;}
       \WHILE {not convergent}
           \STATE {Calculate hidden layers $\{\mathbf{H}^{(l)}\}_{l=1}^{L}$ by Eq.~\eqref{ForwardUpdating};}
           \STATE {Update weights: $\mathbf{\Theta}^{(l+1)} \leftarrow \mathbf{\Theta}^{(l)} - \eta \frac{\partial \mathcal{L}}{\partial \mathbf{\Theta}^{(l)}}$;}
       \ENDWHILE
    \RETURN {The predicted labels: $y_{i}^{*} = \arg\max_{j}\mathbf{H}^{(L)}_{ij}$.}
    \end{algorithmic}
\end{algorithm}

\begin{table}[]
\caption{Dataset statistics.}
\label{datasets}
\resizebox{\columnwidth}{!}{%
\begin{tabular}{l|rrrrr}
\toprule
\multicolumn{1}{l|}{Datasets} & \multicolumn{1}{c}{\#Nodes} & \multicolumn{1}{c}{\#Edges} & \multicolumn{1}{c}{\#Classes} & \multicolumn{1}{c}{\#Features} & \multicolumn{1}{c}{\#Train/Val/Test} \\ \midrule
Cora                           & 2,708                        & 5,429                        & 7                              & 1,433                           & 140/500/1,000                         \\ 
Citeseer                       & 3,327                        & 4,732                        & 6                              & 3,703                           & 120/500/1,000                         \\ 
Pubmed                         & 19,717                       & 44,338                       & 3                              & 500                             & 60/500/1,000                          \\ 
ACM                            & 3,025                        & 13,128                       & 3                              & 1,870                           & 60/500/1,000                          \\ 
BlogCatalog                    & 5,196                        & 171,743                      & 6                              & 8,189                           & 120/500/1,000                         \\ 
CoraFull                       & 19,793                       & 65,311                       & 70                             & 8,710                           & 1,400/500/1,000                       \\ 
Flickr                         & 7,575                        & 239,738                      & 9                              & 12,047                          & 180/500/1,000                         \\ 
UAI                            & 3,067                        & 28,311                       & 19                             & 4,973                           & 367/500/,1000                         \\ \bottomrule
\end{tabular}%
}
\end{table}

\begin{table*}[!tbp]
    \begin{center}
      \caption{{Accuracy} and {F1-score} (mean\% and standard deviation\%) of all methods, where the best results are in red and the second-best are in blue. Note that GraphSAGE fails to work on the ACM dataset, and thus its results are marked with ``---''.}
      \resizebox{0.98\textwidth}{!}{
        \begin{tabular}{m{1.15cm}<{\centering}|m{2.75cm}|m{1.78cm}<{\centering}m{1.78cm}<{\centering}m{1.78cm}<{\centering}m{1.78cm}<{\centering}m{1.78cm}<{\centering}m{1.78cm}<{\centering}m{1.78cm}<{\centering}m{1.78cm}<{\centering}}
            \toprule[1pt]
            Metrics                & \multicolumn{1}{c|}{Methods / Datasets} & Cora                                       & Citeseer                                   & Pubmed                                     & ACM                                        & BlogCatalog                                & CoraFull                                   & Flickr                                     & UAI                                        \\ \midrule
                                    & Chebyshev                             & 76.2 (0.7)                                 & 69.3 (0.4)                                 & 74.0 (0.8)                                 & 82.8 (1.4)                                 & 68.3 (1.6)                                 & 57.2 (1.1)                                 & 38.5 (1.6)                                 & 49.7 (0.4)                                 \\
                                    & GraphSAGE                             & 76.7 (0.6)                                 & 64.4 (0.9)                                 & 75.5 (0.2)                                 & ---                                          & 57.8 (0.7)                                 & 59.9 (0.7)                                 & 32.7 (1.0)                                 & 41.7 (1.4)                                 \\
                                    & GAT                                   & 79.1 (0.8)                                 & 68.3 (0.5)                                 & 78.4 (0.3)                                 & 84.6 (0.5)                                 & 67.1 (1.7)                                 & 62.4 (0.4)                                 & 40.4 (0.9)                                 & 49.7 (3.0)                                 \\ \cmidrule{2-10} 
                                    & GCN                                   & 80.6 (1.4)                                 & 69.1 (1.5)                                 & 77.6 (1.3)                                 & 88.8 (0.5)                                 & 84.2 (0.6)                                 & 62.8 (0.4)                                 & 51.0 (1.2)                                 & 58.5 (1.1)                                 \\
                                    & SGC                                   & 79.3 (1.0)                                 & 66.4 (1.7)                                 & 76.8 (2.0)                                 & 80.8 (2.7)                                 & 81.3 (0.2)                                 & 62.9 (2.2)                                 & 51.0 (0.1)                                 & 56.5 (3.5)                                 \\
                                    & APPNP                                 & 78.0 (0.1)                                 & 65.8 (0.2)                                 & 78.0 (0.0)                                 & 88.2 (0.0)                                 & {\textcolor{blue}{\textbf{87.7 (0.3)}}}                                 & 63.1 (0.5)                                 & 57.5 (0.2)                                 & 62.3 (1.2) \\
                                    & JKNet                                 & {\textcolor{red}{\textbf{83.1 (0.1)}}} & 72.3 (0.1)                                 & 80.1 (0.2)                                 & 82.3 (0.6)                                 & 75.7 (0.1)                                 & 62.6 (0.0)                                 & 54.0 (0.3)                                 & 45.6 (0.5)                                 \\
                                    & DAGNN                                 & 81.9 (0.7)                                 & 70.0 (1.1)                                 & {\textcolor{blue}{\textbf{80.6 (0.7)}}} & 87.4 (0.9)                                 & 84.6 (1.9)                                 & 65.6 (0.3)                                 & 54.6 (5.9)                                 & 46.7 (12.4)                                \\
                                    & GNN-LF                                & 81.1 (0.5)                                 & 72.3 (0.9)                                 & 80.0 (0.4)                                 & 90.8 (0.5)                                 & 86.7 (0.6)                                 & 63.5 (0.9)                                 & 56.6 (0.6)                                 & 36.6 (19.8)                                \\
                                    & GNN-HF                                & 80.7 (0.2)                                 & 68.8 (1.3)                                 & 77.7 (0.2)                                 & {\textcolor{blue}{\textbf{91.2 (0.5)}}} & 84.5 (0.4)                                 & 63.0 (0.7)                                 & {\textcolor{blue}{\textbf{60.7 (0.4)}}} & 54.8 (1.4)                                 \\ \cmidrule{2-10} 
                                    & tsGCN (inv)                          & 80.3 (0.3)                                 & {\textcolor{red}{\textbf{73.3 (0.4)}}} & 78.4 (0.3)                                 & 85.1 (1.6)                                 & 87.8 (6.3) & {\textcolor{blue}{\textbf{67.0 (0.9)}}} & 53.3 (12.6)                                & {\textcolor{blue}{\textbf{64.2 (1.8)}}}                                 \\
            \multirow{-12}{*}{ACC} & tsGCN                                & {\textcolor{blue}{\textbf{82.0 (0.3)}}} & {\textcolor{blue}{\textbf{73.1 (0.4)}}} & {\textcolor{red}{\textbf{82.4 (0.1)}}} & {\textcolor{red}{\textbf{92.8 (0.3)}}} & {\textcolor{red}{\textbf{92.3 (0.5)}}} & {\textcolor{red}{\textbf{67.9 (0.9)}}} & {\textcolor{red}{\textbf{79.1 (3.0)}}} & {\textcolor{red}{\textbf{67.9 (0.6)}}} \\ \midrule
                                    & Chebyshev                             & 76.3 (0.7)                                 & 65.4 (0.8)                                 & 73.9 (0.7)                                 & 82.5 (1.4)                                 & 64.3 (1.6)                                 & 40.0 (0.5)                                 & 38.4 (1.5)                                 & 39.1 (0.2)                                 \\
                                    & GraphSAGE                             & 76.7 (0.5)                                 & 60.7 (0.5)                                 & 74.7 (0.2)                                 & ---                                          & 54.7 (0.6)                                 & 51.9 (0.6)                                 & 31.0 (1.1)                                 & 35.3 (1.0)                                 \\
                                    & GAT                                   & 77.1 (0.7)                                 & 64.6 (0.5)                                 & 78.2 (0.2)                                 & 84.8 (0.5)                                 & 66.3 (1.9)                                 & 46.4 (0.4)                                 & 38.1 (1.1)                                 & 40.8 (1.3)                                 \\ \cmidrule{2-10} 
                                    & GCN                                   & 79.4 (1.4)                                 & 65.2 (2.4)                                 & 77.2 (1.4)                                 & 88.9 (0.5)                                 & 82.4 (0.5)                                 & 52.8 (0.8)                                 & 50.0 (1.7)                                 & 45.0 (1.1)                                 \\
                                    & SGC                                   & 77.7 (0.9)                                 & 61.5 (1.7)                                 & 76.5 (2.3)                                 & 81.1 (2.6)                                 & 80.7 (0.3)                                 & 53.2 (2.1)                                 & 44.2 (0.2)                                 & 46.7 (1.7)                                 \\
                                    & APPNP                                 & 77.6 (0.1)                                 & 63.2 (0.2)                                 & 77.7 (0.0)                                 & 88.3 (0.0)                                 & 85.7 (0.3)                                 & 48.2 (0.7)                                 & 56.9 (0.2)                                 & {\textcolor{blue}{\textbf{48.6 (1.6)}}} \\
                                    & JKNet                                 & {\textcolor{red}{\textbf{82.3 (0.3)}}} & 67.8 (0.1)                                 & 79.3 (0.3)                                 & 82.2 (0.6)                                 & 75.0 (0.1)                                 & 51.3 (0.1)                                 & 51.1 (0.5)                                 & 31.7 (1.5)                                 \\
                                    & DAGNN                                 & 80.0 (0.7)                                 & 65.7 (0.7)                                 & {\textcolor{blue}{\textbf{80.7 (0.7)}}} & 87.5 (0.9)                                 & 83.8 (2.4)                                 & 53.0 (0.9)                                 & 55.5 (6.7)                                 & 39.3 (11.2)                                \\
                                    & GNN-LF                                & 79.1 (0.7)                                 & 66.7 (0.4)                                 & 80.2 (0.5)                                 & 90.9 (0.5)                                 & {\textcolor{blue}{\textbf{85.9 (0.6)}}} & 50.5 (1.9)                                 & 54.3 (1.0)                                 & 29.7 (15.1)                                \\
                                    & GNN-HF                                & 78.6 (0.3)                                 & 64.3 (1.7)                                 & 78.1 (0.2)                                 & {\textcolor{blue}{\textbf{91.3 (0.5)}}} & 83.8 (0.4)                                 & 49.0 (1.1)                                 & {\textcolor{blue}{\textbf{58.6 (0.6)}}} & 44.9 (0.8)                                 \\ \cmidrule{2-10} 
                                    & tsGCN (inv)                          & 78.5 (0.3)                                 & {\textcolor{red}{\textbf{69.6 (0.4)}}} & 78.7 (0.3)                                 & 85.1 (1.5)                                 & 85.2 (7.1)                                 & {\textcolor{blue}{\textbf{57.2 (1.1)}}} & 52.9 (15.8)                                & 48.5 (0.8)                                 \\
            \multirow{-12}{*}{F1}   & tsGCN                                & {\textcolor{blue}{\textbf{80.5 (0.5)}}} & {\textcolor{blue}{\textbf{69.0 (0.3)}}} & {\textcolor{red}{\textbf{82.4 (0.1)}}} & {\textcolor{red}{\textbf{92.8 (0.4)}}} & {\textcolor{red}{\textbf{90.1 (0.6)}}} & {\textcolor{red}{\textbf{58.7 (0.7)}}} & {\textcolor{red}{\textbf{79.3 (2.9)}}} & {\textcolor{red}{\textbf{50.1 (0.1)}}} \\ \bottomrule[1pt]
        \end{tabular}}\label{Performance}
    \end{center}
  \end{table*}
 
\section{tsGCN: Topological and Semantic Regularized Graph Convolutional Network}

One finding from most existing GCNs is that they often ignored feature-based semantic structures, which can weaken the representation learning abilities of graph networks, 
then we focus on two regularizers, i.e.,
\begin{align}
\mathcal{L}_{1}(\mathbf{H}^{(l)};\mathcal{G}) =  \frac{1}{2}\text{Tr}\left(\{\mathbf{H}^{(l)}\}^{\top}(\frac{1}{2}\mathbf{I} + \alpha\widetilde{\mathbf{L}}_{\mathcal{G}})\mathbf{H}^{(l)}\right), \\
\mathcal{L}_{2}(\mathbf{H}^{(l)}; \mathcal{X}) =  \frac{1}{2}\text{Tr}\left(\{\mathbf{H}^{(l)}\}^{\top}(\frac{1}{2}\mathbf{I} + \beta\widetilde{\mathbf{L}}_{\mathcal{X}})\mathbf{H}^{(l)}\right), 
\end{align}
where $\widetilde{\mathbf{L}}_{\mathcal{G}}$ is a graph Laplacian from the given adjacency matrix (e.g., $\widetilde{\mathbf{L}}_{\mathcal{G}}=\widetilde{\mathbf{L}}$), and $\widetilde{\mathbf{L}}_{\mathcal{X}}$ is a graph Laplacian calculated from the pairwise similarity of any two graph nodes.  
Hence, we devise a dual-regularizer, i.e., $\mathcal{L}(\mathbf{H}^{(l)})=\mathcal{L}_{1}(\mathbf{H}^{(l)};\mathcal{G})+\mathcal{L}_{2}(\mathbf{H}^{(l)};\mathcal{X})$, and if it is under the optimization framework (\ref{regularizer_framework}), then there yields the following updating rule
\begin{align}\label{ForwardUpdating}
\mathbf{H}^{(l)} = \sigma\left((\mathbf{I} + \alpha\widetilde{\mathbf{L}}_{\mathcal{G}} + \beta\widetilde{\mathbf{L}}_{\mathcal{X}})^{-1}\mathbf{H}^{(l-1)}\mathbf{\Theta}^{(l)}\right).
\end{align}
Since this method seeks to preserve both the topological and semantic structures for more accurate presentations, we call it tsGCN (i.e., \textbf{T}opological and \textbf{S}emantic regularized \textbf{GCN}).

Notably, the computational complexity of $(\mathbf{I} +  \alpha\widetilde{\mathbf{L}}_{\mathcal{G}} + \beta\widetilde{\mathbf{L}}_{\mathcal{X}})^{-1}$ is $\mathcal{O}(N^{3})$, which tends to be unaffordable in practical applications. 
To this end, a low-rank approximation is operated, i.e.,   $\alpha\widetilde{\mathbf{L}}_{\mathcal{G}} + \beta\widetilde{\mathbf{L}}_{\mathcal{X}} \approx \mathbf{W}\mathbf{V}^{\top}$, where $\mathbf{W}, \mathbf{V}\in\mathbb{R}^{N\times r}$ with $r \ll N$. 
This leads to the Woodbury matrix identity:
\begin{align}
 (\mathbf{I} +  \mathbf{W}\mathbf{V}^{\top})^{-1} =  \mathbf{I} - \mathbf{W}(\mathbf{I} +  \mathbf{V}^{\top}\mathbf{W})^{-1}\mathbf{V}^{\top}, 
\end{align}
of which the computational complexity costs $\mathcal{O}(N^{2})$. 
 
Given that the optimal $\mathbf{M}^{*}$ of the following problem 
\begin{align}\label{label_rank_approximation}
\min_{\mathbf{M}\in\mathbb{R}^{N\times N}:~\text{rank}(\mathbf{M})=r}||\mathbf{M} - (\alpha\widetilde{\mathbf{L}}_{\mathcal{G}} + \beta\widetilde{\mathbf{L}}_{\mathcal{X}})||_{\textsc{F}}^{2}
\end{align}
is attained at the $r$-truncated singular value decomposition of $\alpha\widetilde{\mathbf{L}}_{\mathcal{G}} + \beta\widetilde{\mathbf{L}}_{\mathcal{X}}$, i.e., $\mathbf{M}^{*} = \mathbf{U}\mathbf{\Sigma}\mathbf{U}^{\top}$, where 
$\mathbf{\Sigma}\in\mathbb{R}^{r\times r}$ is a diagonal matrix containing the $r$ largest singular values. 
An optimal $\{\mathbf{W}^{*}, \mathbf{V}^{*}\}$ to $\alpha\widetilde{\mathbf{L}}_{\mathcal{G}} + \beta\widetilde{\mathbf{L}}_{\mathcal{X}} \approx \mathbf{W}\mathbf{V}^{\top}$ can be given by an analytic form of $\mathbf{W}^{*}=\mathbf{V}^{*}=\mathbf{U}\mathbf{\Sigma}^{\frac{1}{2}}$. 

To obtain the optimum $\{\mathbf{W}^{*}, \mathbf{V}^{*}\}$, the iterative algorithm \cite{SunXu2019Neural} with $\mathcal{O}(N^{2})$ is leveraged as
\begin{equation}
   \mathbf{Z}^{(t+1)} \leftarrow (\alpha\widetilde{\mathbf{L}}_{\mathcal{G}} + \beta\widetilde{\mathbf{L}}_{\mathcal{X}})\mathbf{U}^{(t)}, 
\end{equation}
\begin{equation}
    \{\mathbf{U}^{(t+1)}, \mathbf{R}^{(t+1)}\} \leftarrow \text{QR} (\mathbf{Z}^{(t+1)}),
\end{equation}
where QR$(\cdot)$ denotes the QR-decomposition. 
Note that this algorithm can converge to the $r$ largest eigenvalues $\mathbf{R}^{(t+1)}$ and its corresponding eigenvectors $\mathbf{Z}^{(t+1)}$ when the iterative number $t$ is large enough. 
Finally, there will be $\mathbf{W}^{*} = \mathbf{V}^{*} = \mathbf{U}^{(t+1)}[\mathbf{R}^{(t+1)}]^{\frac{1}{2}}$. 

Gathering all analyses mentioned above, the procedure for tsGCN is summarized in Algorithm \ref{alg:Framework}. 

\section{Experiment}
This section will show tsGCN's effectiveness and efficiency via comprehensive experiments.

\subsection{Datasets}

Cora, Citeseer and Pubmed are citation networks, and CoraFull is a larger version of Cora;
ACM is a paper network, and BlogCatalog and Flickr are social networks;
UAI has been utilized for community detection.
The detailed statistics of the above eight public datasets are concluded in Table \ref{datasets}.

\subsection{Compared Methods}
Two types of methods are employed here for comparisons. Chebyshev \cite{DefferrardBV16}, GraphSAGE \cite{HamiltonYL17} and GAT \cite{VelickovicCCRLB18} are classical graph neural networks. GCN, SGC \cite{WuJrZhang2019Simplifying}, APPNP \cite{KlicperaBojchevski2019Predict}, JKNet \cite{XuLiTian2018Representation}, DAGNN \cite{LiuGaoJi2020Towards}, GNN-LF and GNN-HF \cite{ZhuWang2021Interpreting} are selected as state-of-the-art GCN variants.


\begin{figure*}[!ht]
    \centering
	\begin{subfigure}[Runtime\label{fig:a}]{
	    \centering
		\includegraphics[width=0.345\textwidth]{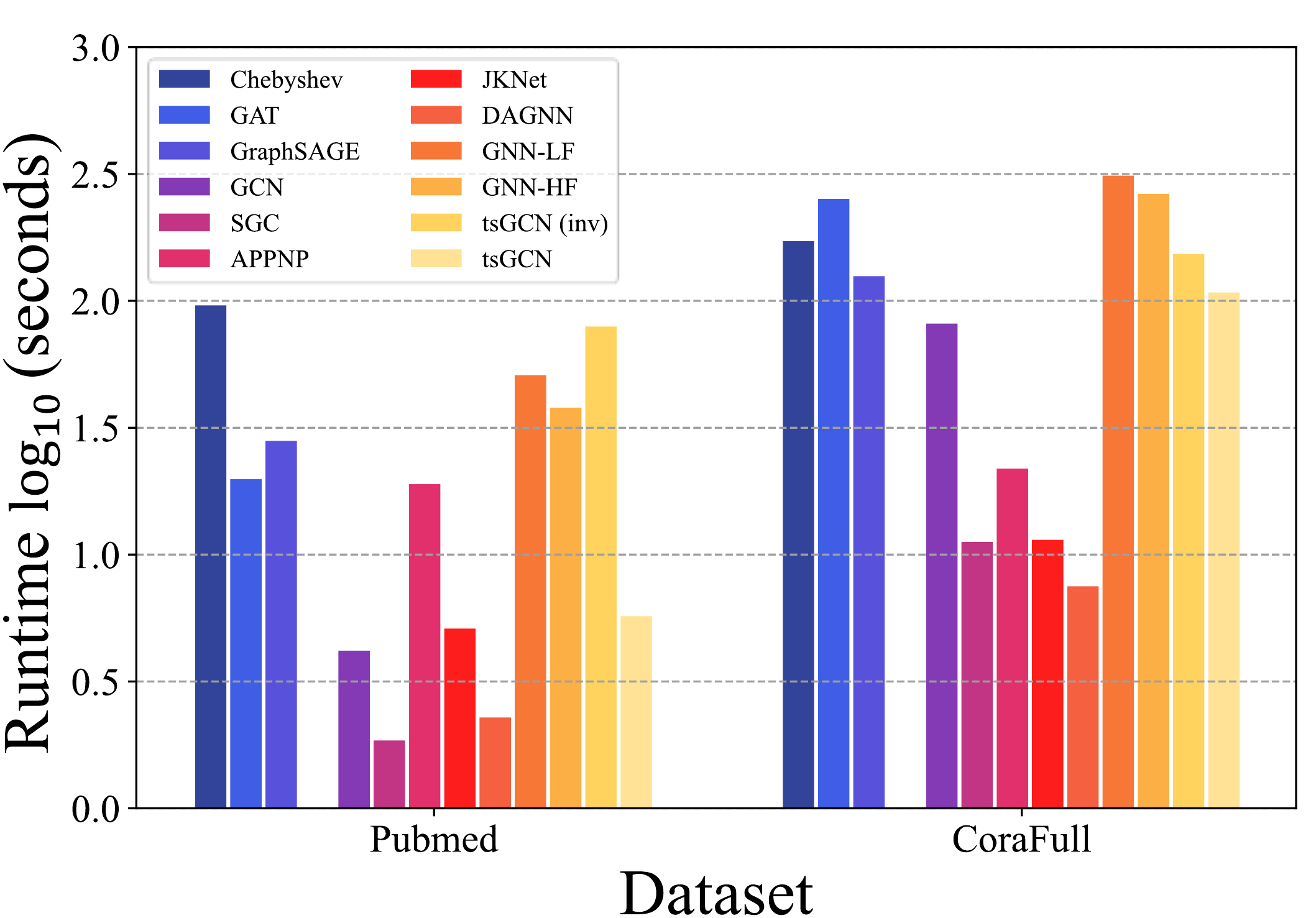}}
	\end{subfigure}\hspace{20mm} 
	\begin{subfigure}[Classification Accuracy\label{fig:b}]{
	    \centering
		\includegraphics[width=0.345\textwidth]{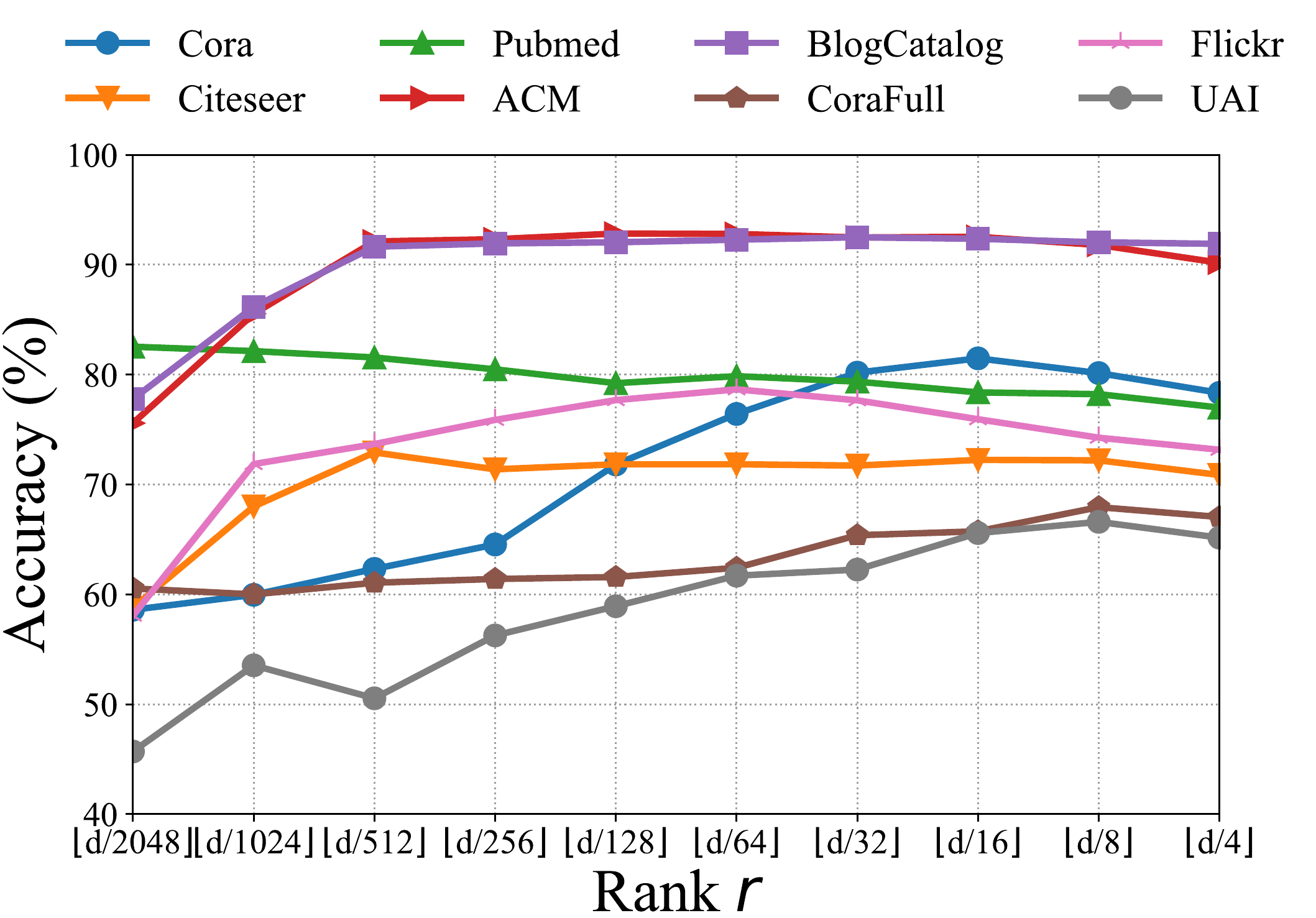}}
	\end{subfigure}
	\caption{(a) All methods' runtime on two large datasets. (b) The classification accuracy of tsGCN w.r.t. ($\alpha$, $\beta$) on all datasets.}
	\label{Runtime}
\end{figure*}


\begin{figure*}[ht]
	\centering
	\subfigbottomskip=-3pt
	\subfigcapskip=-1.5pt
	\begin{subfigure}[Cora]{
	    \centering
		\includegraphics[width=0.2\textwidth]{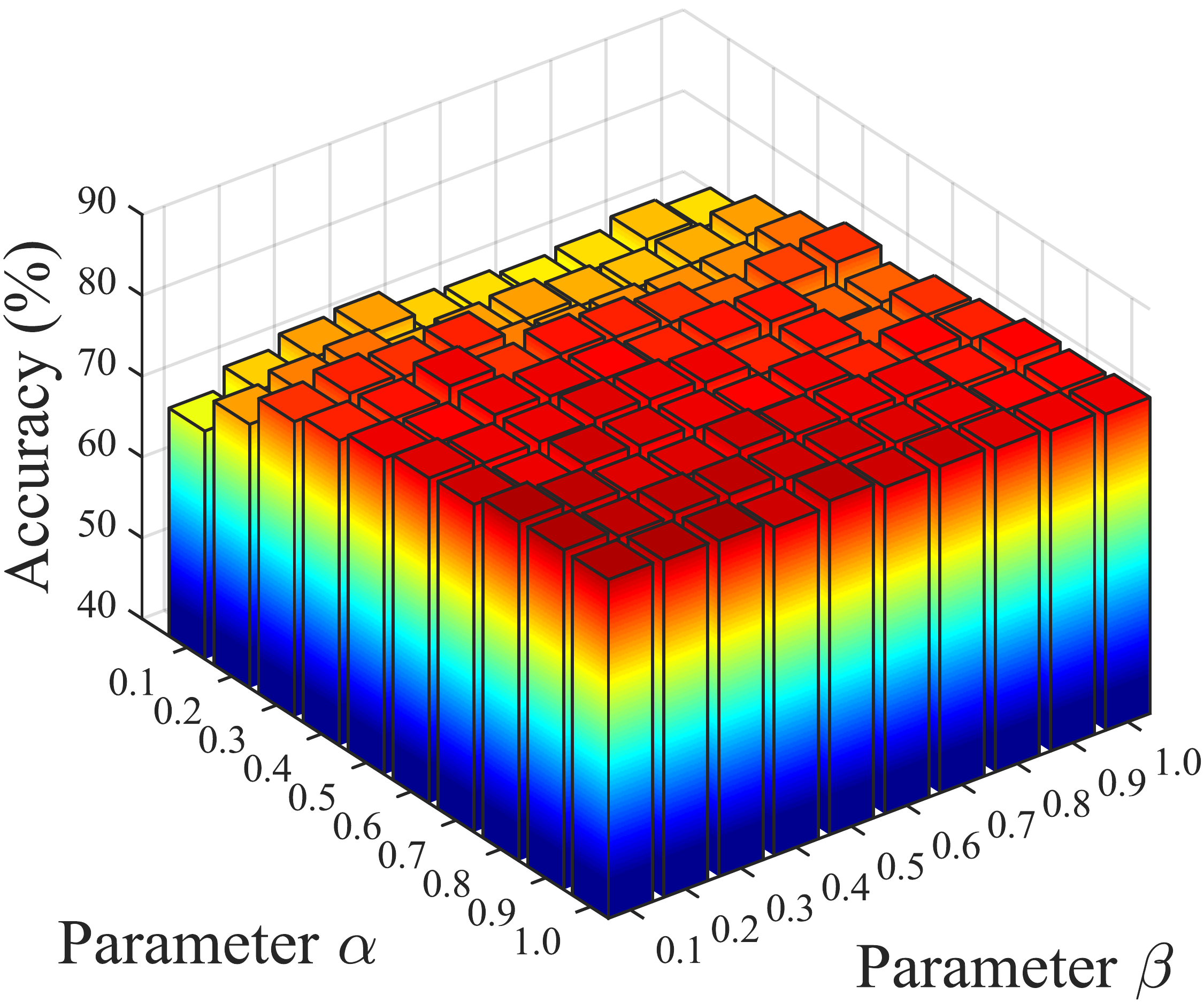}}
	\end{subfigure}\hspace{7mm}
	\begin{subfigure}[Citeseer]{
	    \centering
		\includegraphics[width=0.2\textwidth]{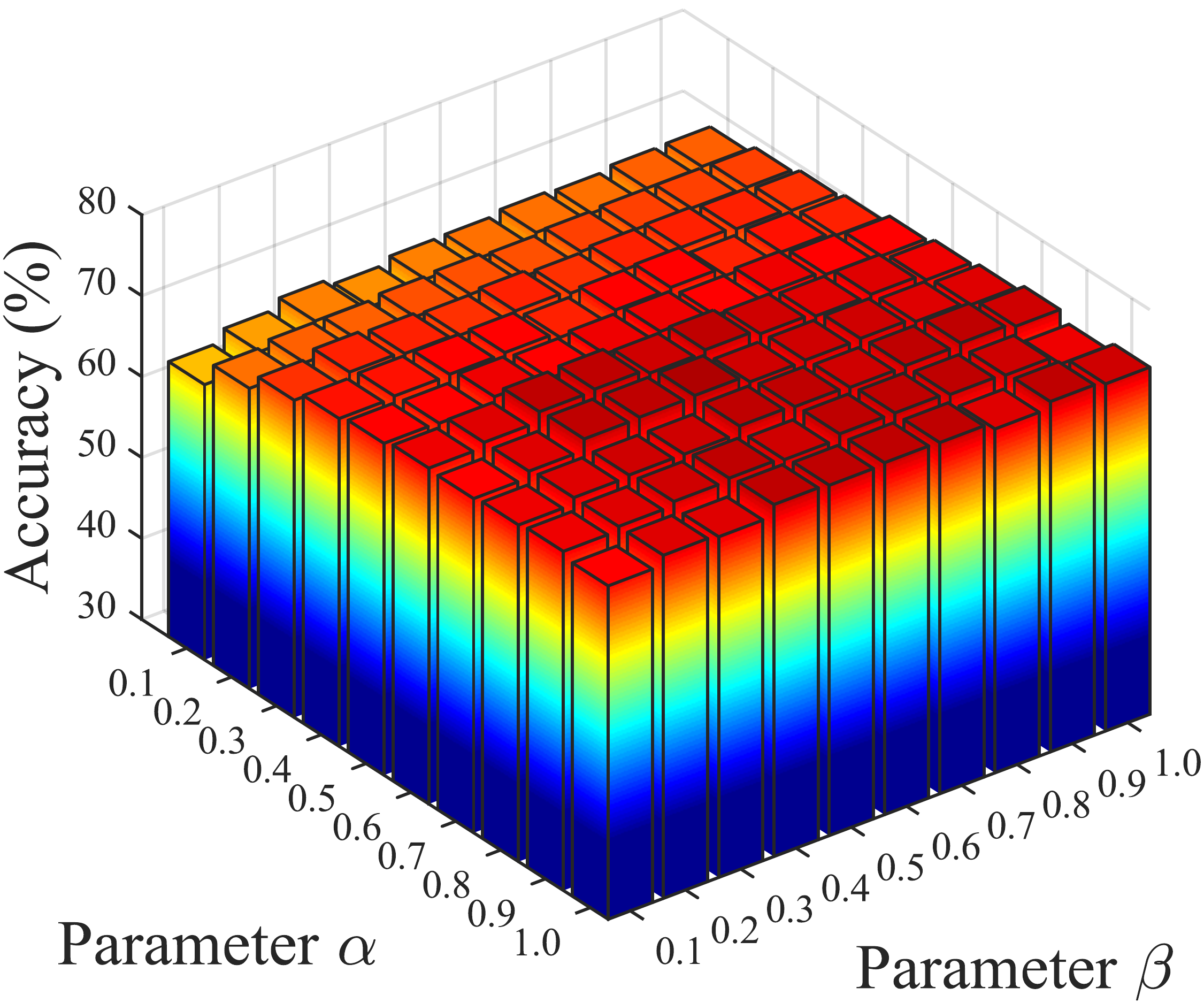}}
	\end{subfigure}\hspace{7mm}
	\begin{subfigure}[Pubmed]{
	    \centering
		\includegraphics[width=0.2\textwidth]{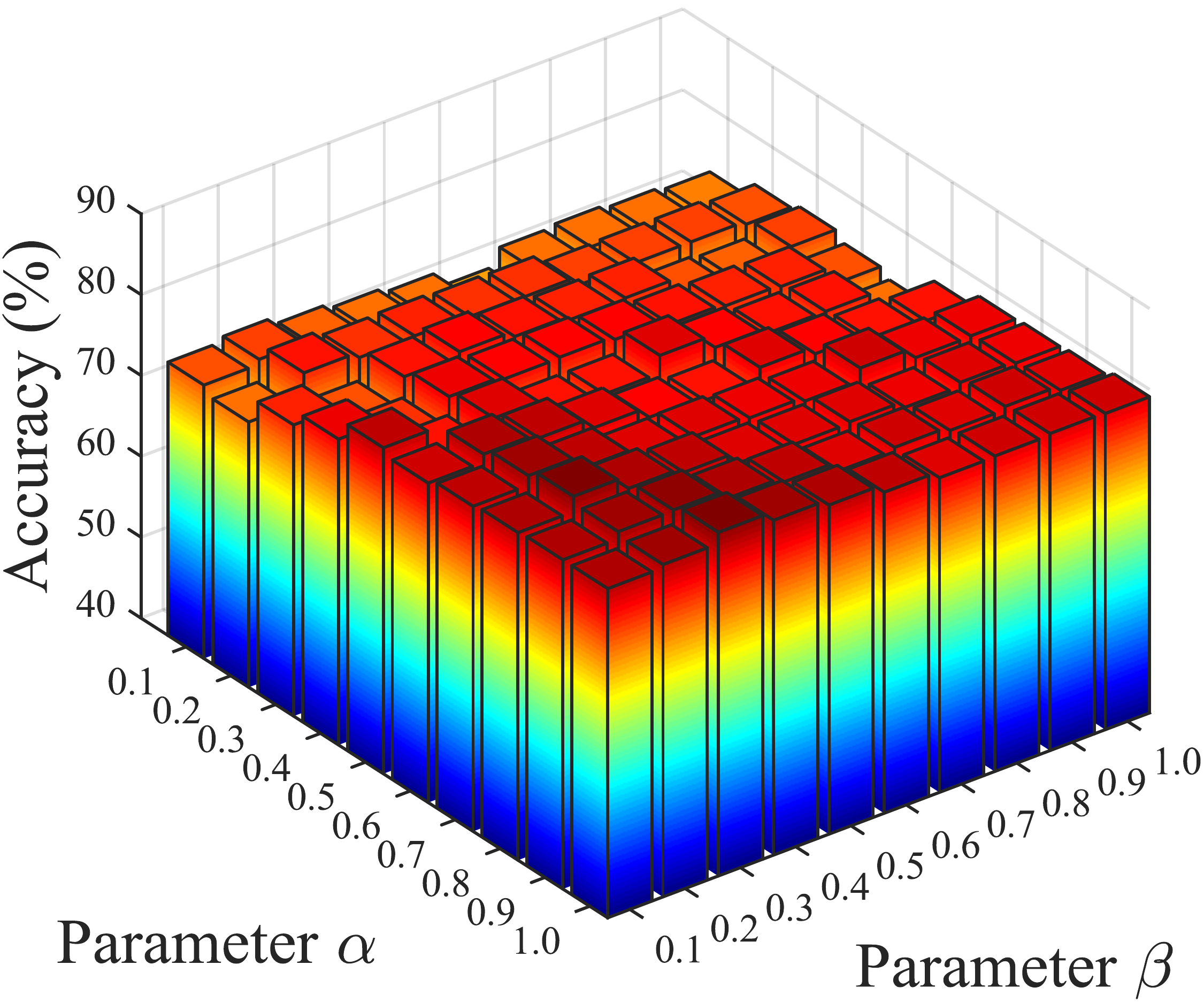}}
	\end{subfigure}\hspace{7mm}
	\begin{subfigure}[ACM]{
	    \centering
		\includegraphics[width=0.2\textwidth]{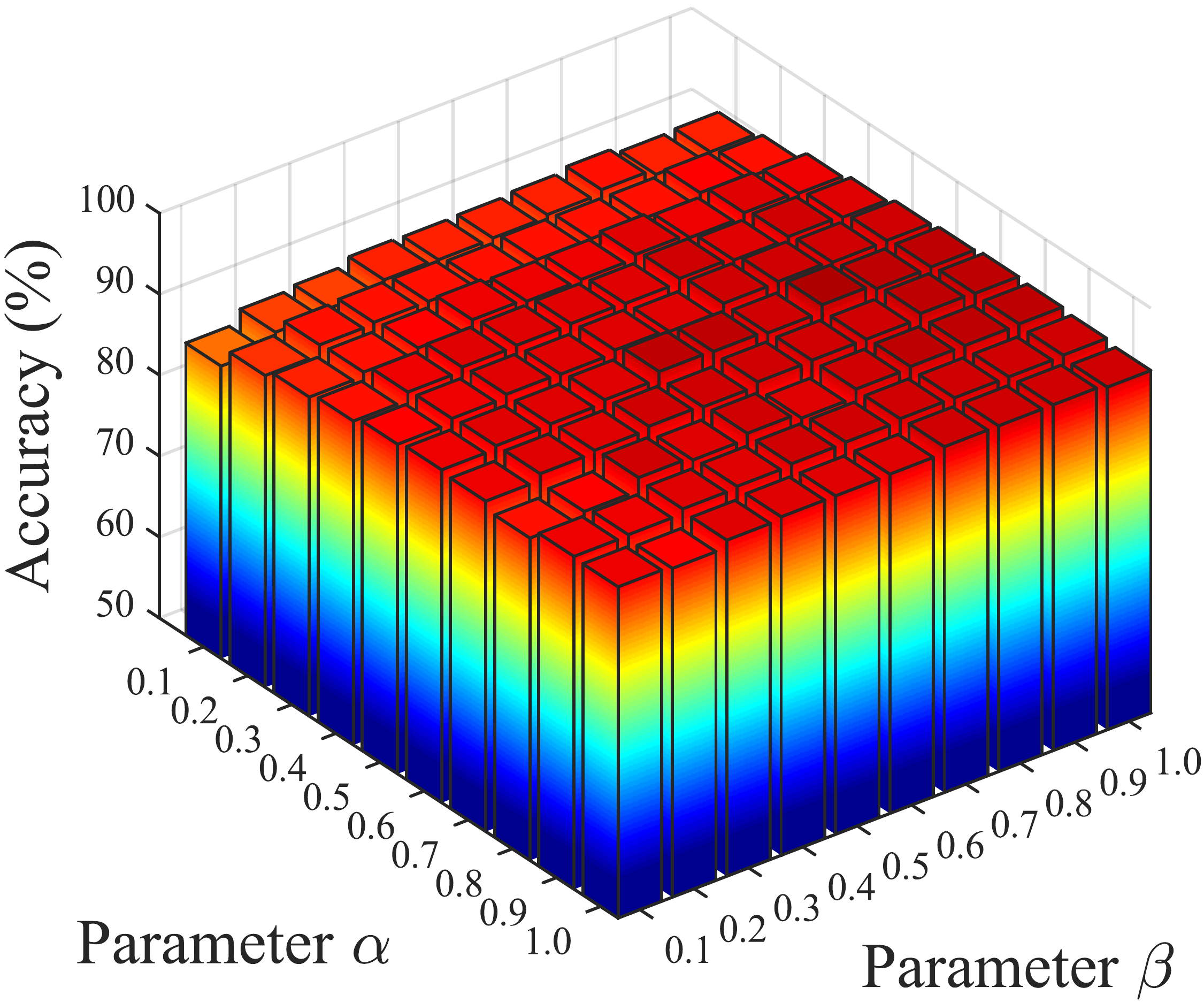}}
	\end{subfigure}\\
	\begin{subfigure}[BlogCatalog]{
	    \centering
		\includegraphics[width=0.2\textwidth]{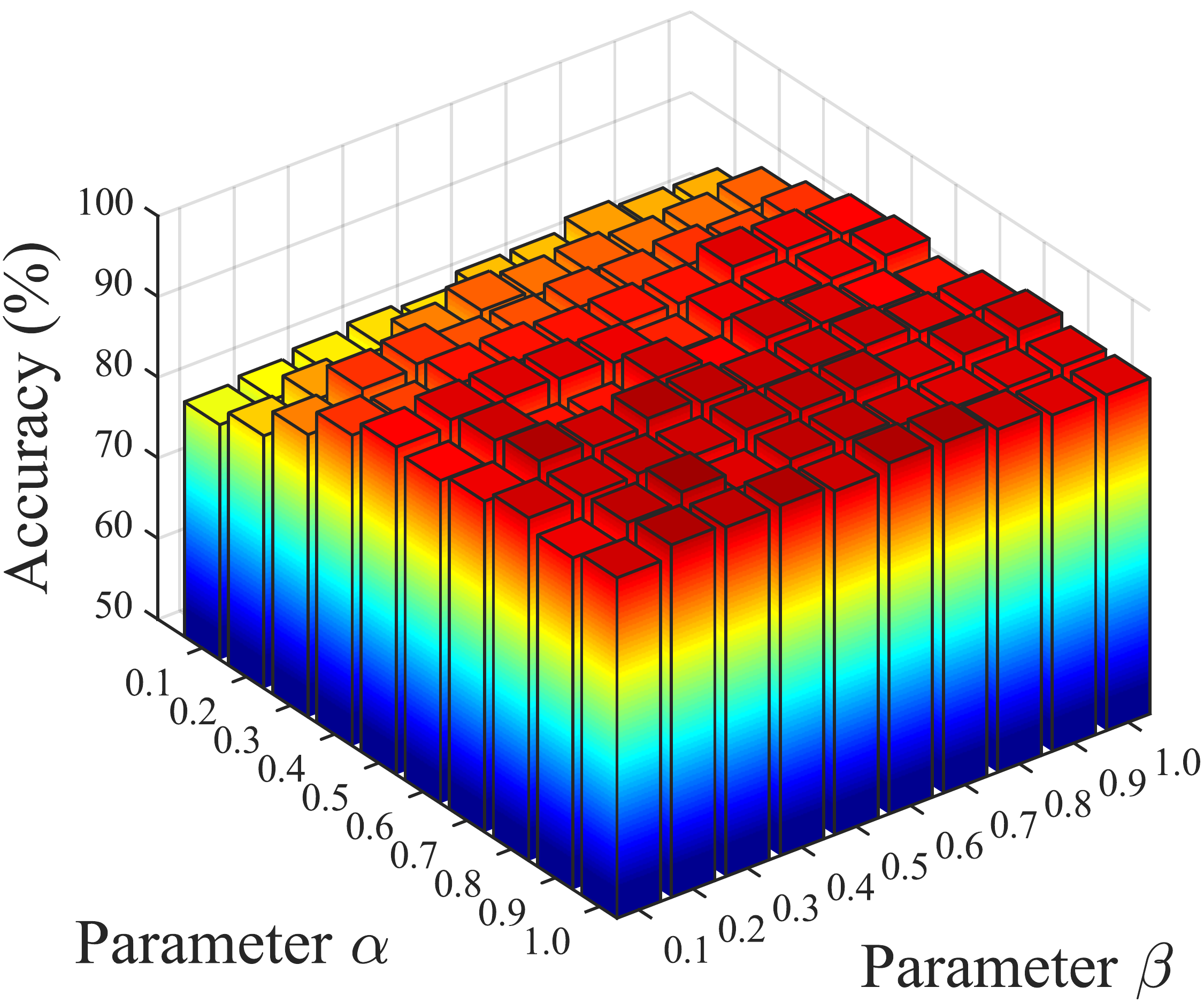}}
	\end{subfigure}\hspace{7mm}
	\begin{subfigure}[CoraFull]{
	    \centering
		\includegraphics[width=0.2\textwidth]{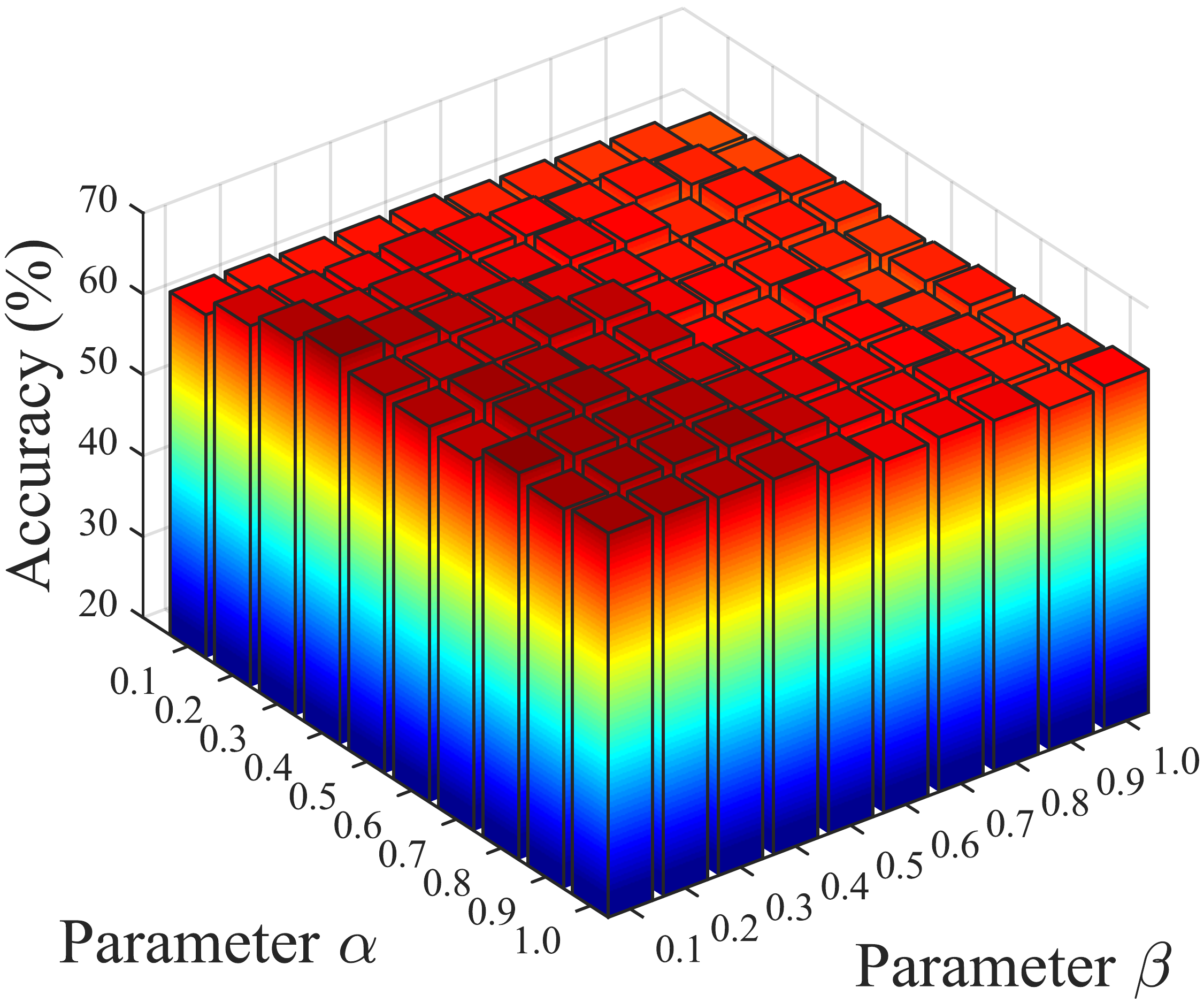}}
	\end{subfigure}\hspace{7mm}
	\begin{subfigure}[Flickr]{
	    \centering
		\includegraphics[width=0.2\textwidth]{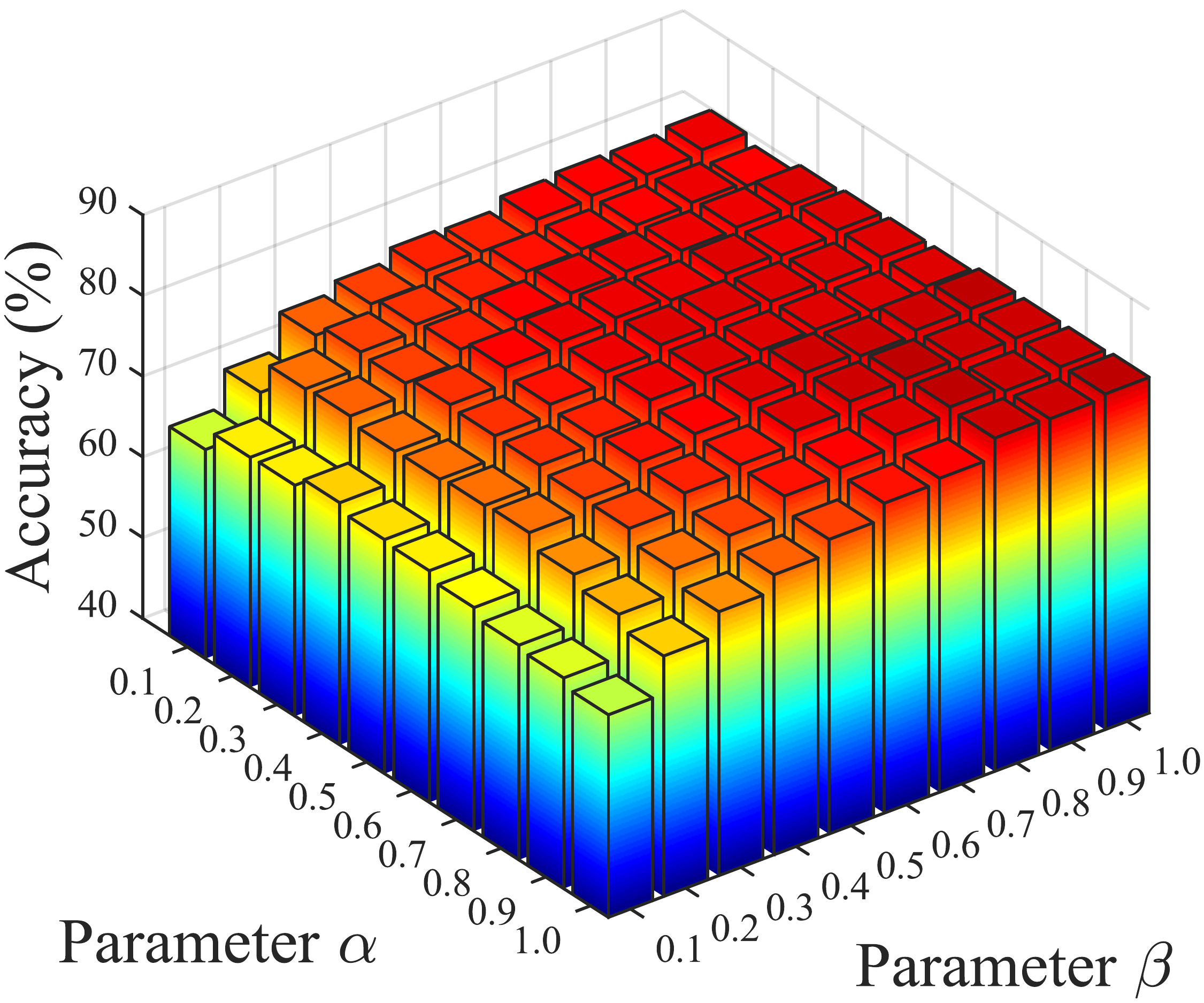}}
	\end{subfigure}\hspace{7mm}
	\begin{subfigure}[UAI]{
	    \centering
		\includegraphics[width=0.2\textwidth]{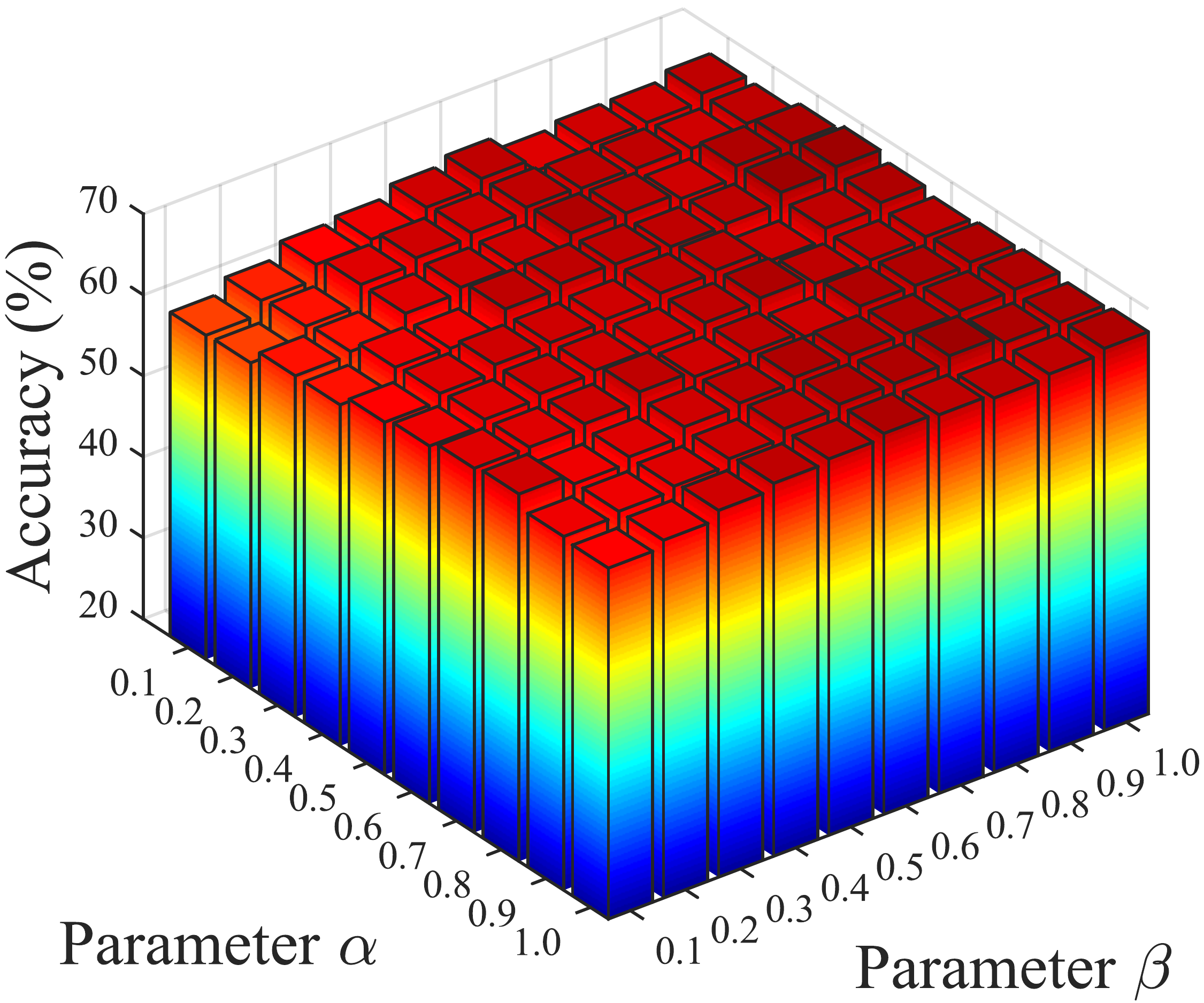}}
	\end{subfigure}
	\caption{The classification accuracy of tsGCN w.r.t. hyperparameters $\alpha$ and $\beta$ on all datasets. 
 	}
	\label{Sensitivity}
\end{figure*}

\subsection{Parameter Setups}
For all experiments, we randomly split samples into a small set of $20$ labeled samples per class for training, a set of $500$ samples for validating and a set of $1,000$ samples for testing. 
In terms of the ten baseline methods, all their configurations are set as the default in their original papers.
With respect to tsGCN, following the vanilla GCN, the learning rate, weight decay and the size of hidden units are set to $1 \times 10^{-2}$, $5 \times 10^{-4}$ and $32$, respectively.
The hyperparameters $\alpha$ and $\beta$ are selected in $\{0.1, 0.2, \ldots, 1.0\}$ for different datasets, and $r$ is chosen in $\{\lfloor \frac{d}{2^{11}} \rfloor, \lfloor \frac{d}{2^{10}} \rfloor, \ldots, \lfloor \frac{d}{2^{3}} \rfloor\}$, where $d$ is the feature dimension of the original data.

\subsection{Semi-supervised Classification}

\textbf{Performance Comparisons.} The semi-supervised classification task is conducted on selected datasets, whose results are recorded in Table \ref{Performance}.
Specifically, we compare our tsGCN with the ten baseline methods in terms of both accuracy and F1-score, marking the best and second-best results on each dataset.
Note that tsGCN (inv) denotes tsGCN without the low-rank approximation, which directly calculates the matrix inverse in Eq.~(\ref{ForwardUpdating}). From Table \ref{Performance}, we have the following observations:
\begin{itemize}
    \item tsGCN achieves the best performances on most datasets, and is only slightly inferior to the JKNet method on the smallest Cora dataset.
    
    \item tsGCN yields higher scores than JKNet and APPNP, especially on Pubmed, CoraFull, BlogCatalog, and Flickr, where the first two are relatively large datasets and the latter two have dense edges.
    tsGCN even outperforms the second-best approach GNN-HF by about $20\%$ on Flickr.
\end{itemize}

It is worth mentioning that tsGCN utilizes high-order information by the infinite-order graph convolution, and JKNet and APPNP also develop different techniques for the same goal.
Hence, the advantage of tsGCN over APPNP and JKNet implies that the infinite-order graph convolution implemented by the low-rank approximation not only requires less computational complexity, but also effectively captures high-order neighborhood information and filters significant noises.

\textbf{Runtime Comparisons.}
This section collects the training time (i.e., runtime) of all methods on two selected large datasets, i.e., Pubmed and CoraFull, as exhibited in Fig.~\ref{fig:a}: the first three columns correspond to classical GNNs, while the rest are GCNs.
From Fig.~\ref{fig:a}, we find that tsGCN takes much less runtime than Chebyshev, GAT, and GraphSAGE; however, it performs moderately well among the state-of-the-art GCN variants.
Specifically, tsGCN is (1) inferior to SGC, JKNet, and DAGNN; (2) well-matched with the original GCN; (3) but advantageous over the recently proposed GNN-LF and GNN-HF.

\begin{figure*}[ht]
	\centering
	\includegraphics[width=0.99\textwidth]{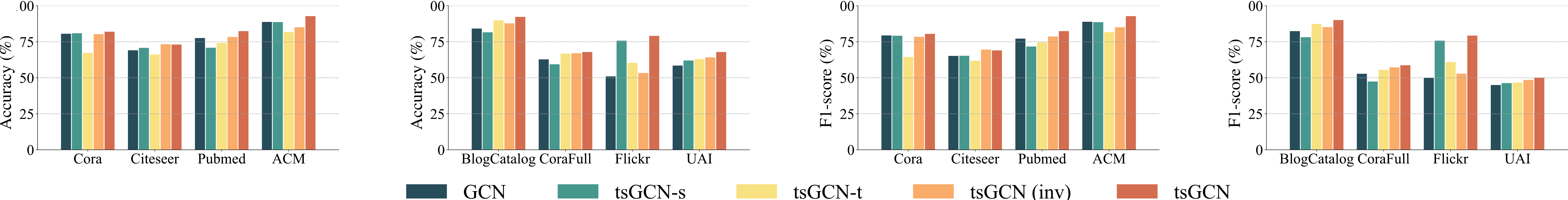}\\
	\caption{{Accuracy} and {F1-score} of tsGCN and its variants on all datasets.}
	\label{Ablation_2}
\end{figure*}

\begin{figure*}[ht]
	\centering
	\subfigbottomskip=-2pt
	\subfigcapskip=-1.5pt
	\begin{subfigure}[Chebyshev]{
	    \centering
		\includegraphics[width=0.13\textwidth]{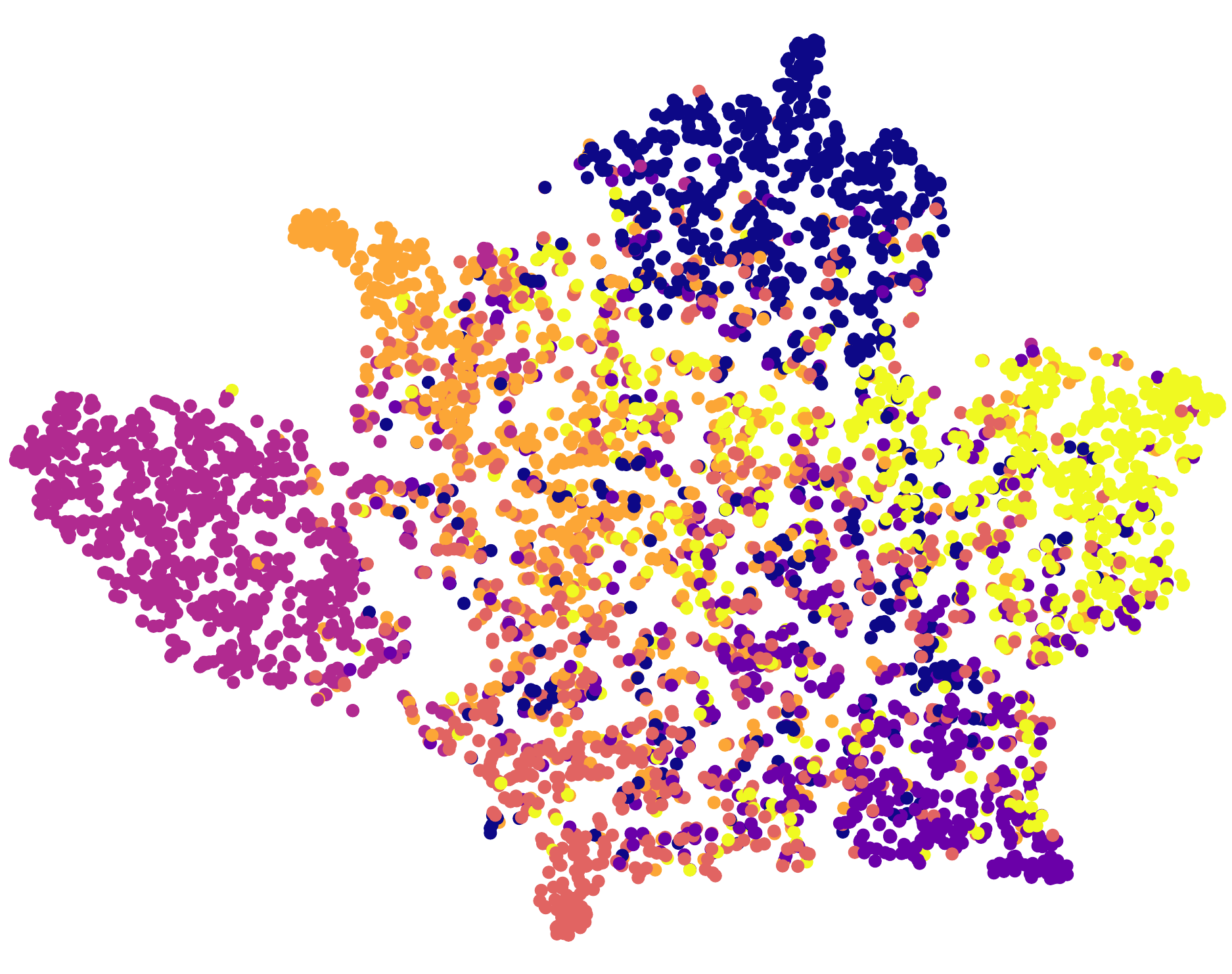}}
	\end{subfigure}\hspace{3.5mm}
	\begin{subfigure}[GraphSAGE]{
	    \centering
		\includegraphics[width=0.13\textwidth]{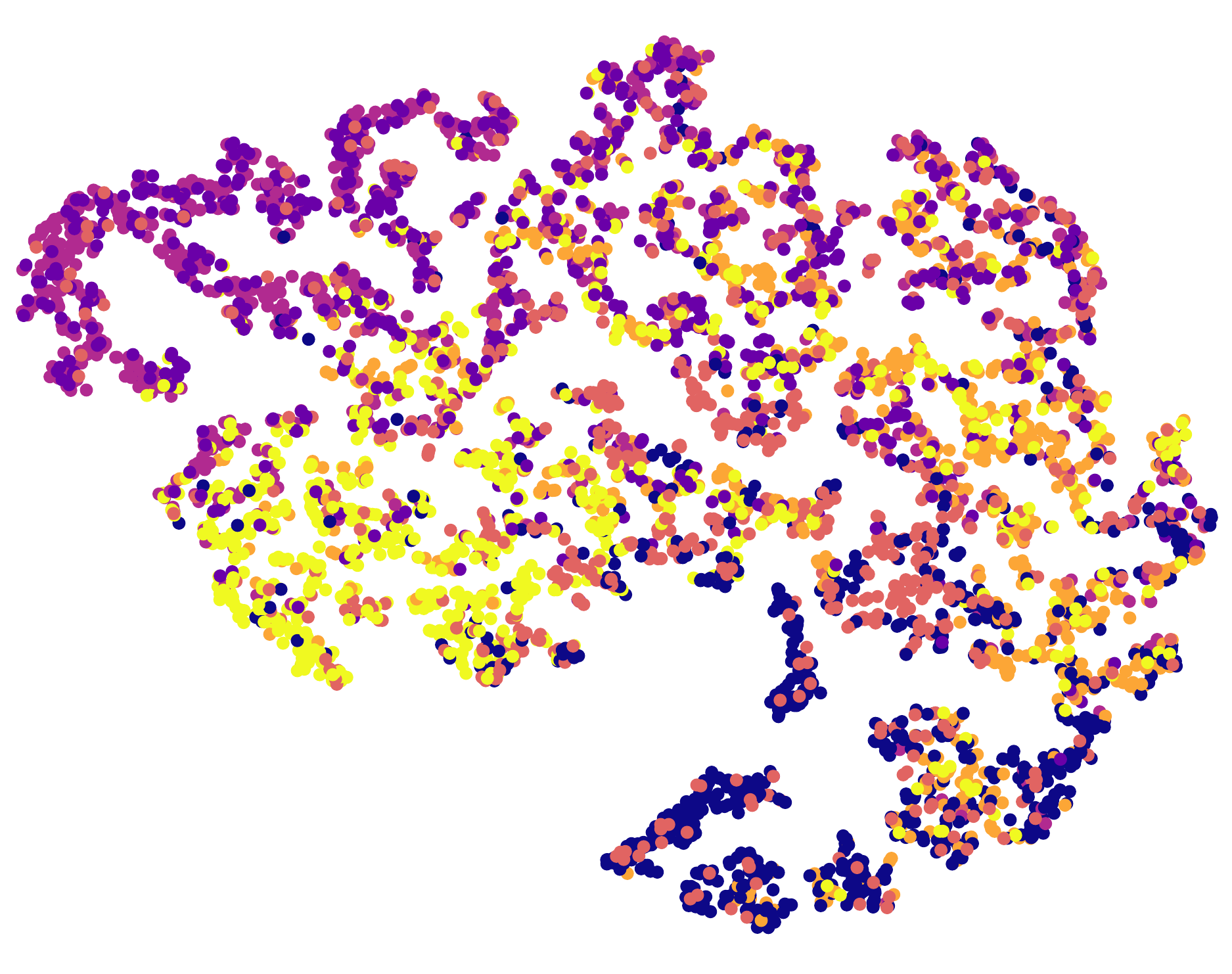}}
	\end{subfigure}\hspace{3.5mm}
	\begin{subfigure}[GAT]{
	    \centering
		\includegraphics[width=0.13\textwidth]{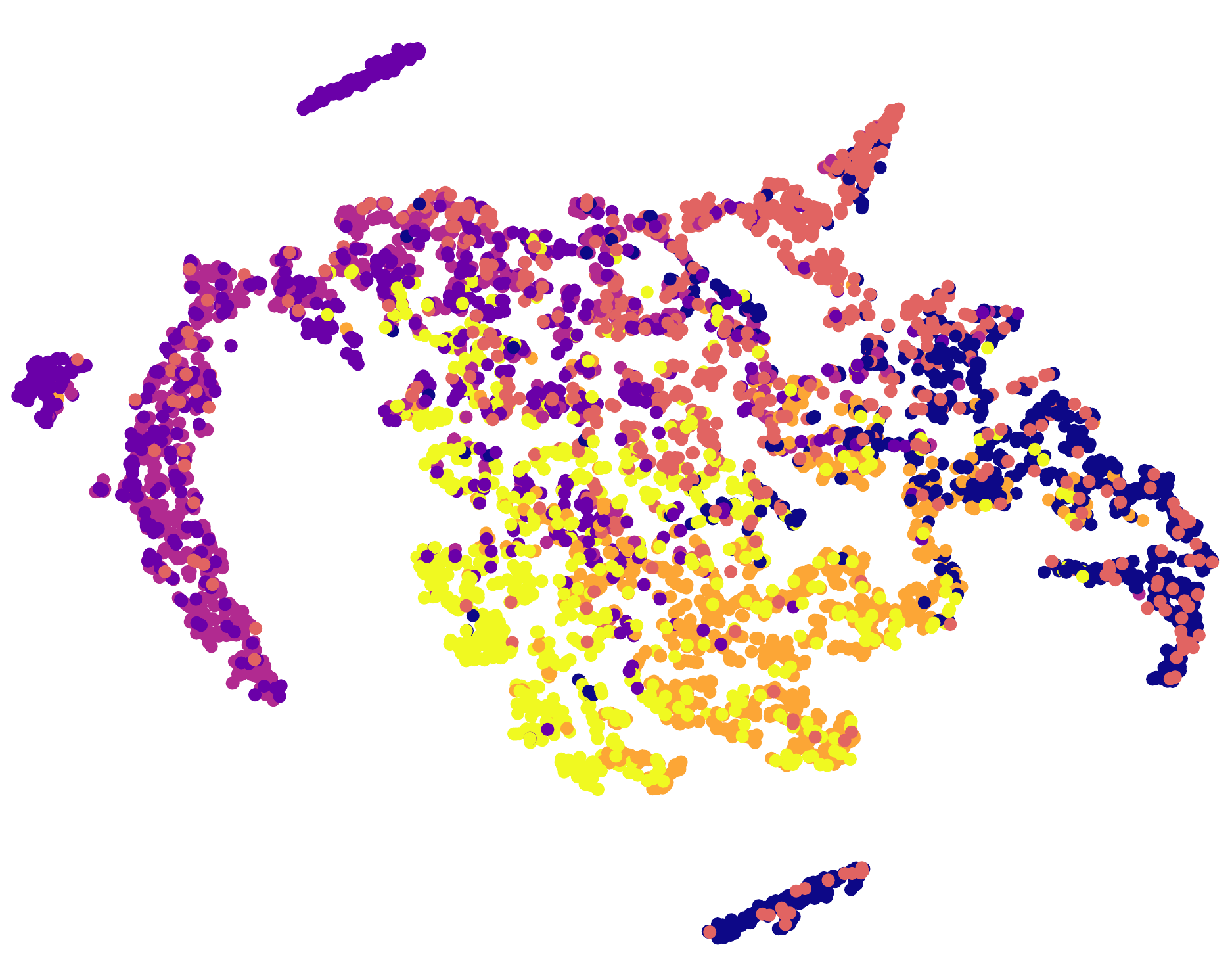}}
	\end{subfigure}\hspace{3.5mm}
	\begin{subfigure}[GCN]{
	    \centering
		\includegraphics[width=0.13\textwidth]{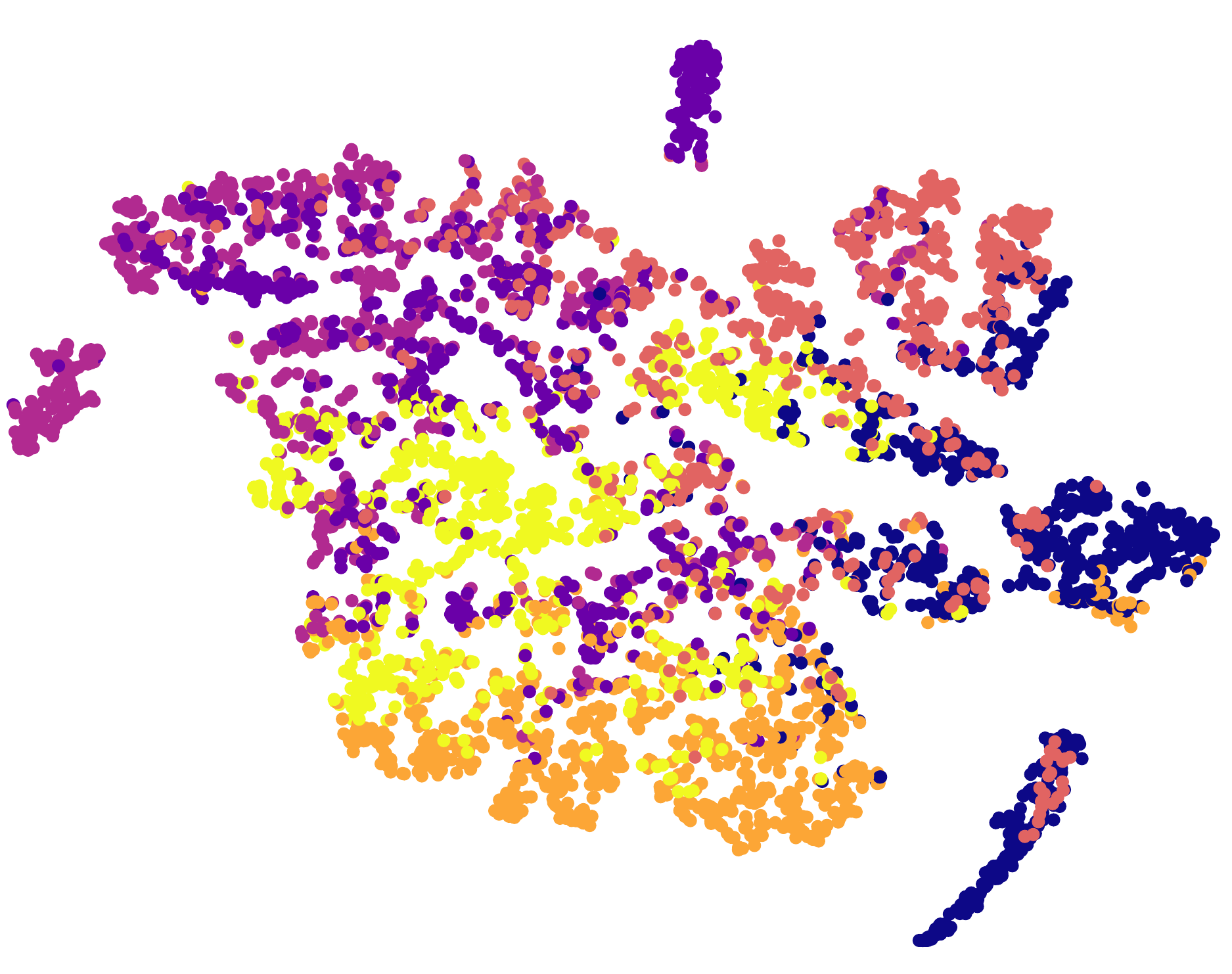}}
	\end{subfigure}\hspace{3.5mm}
	\begin{subfigure}[SGC]{
	    \centering
		\includegraphics[width=0.13\textwidth]{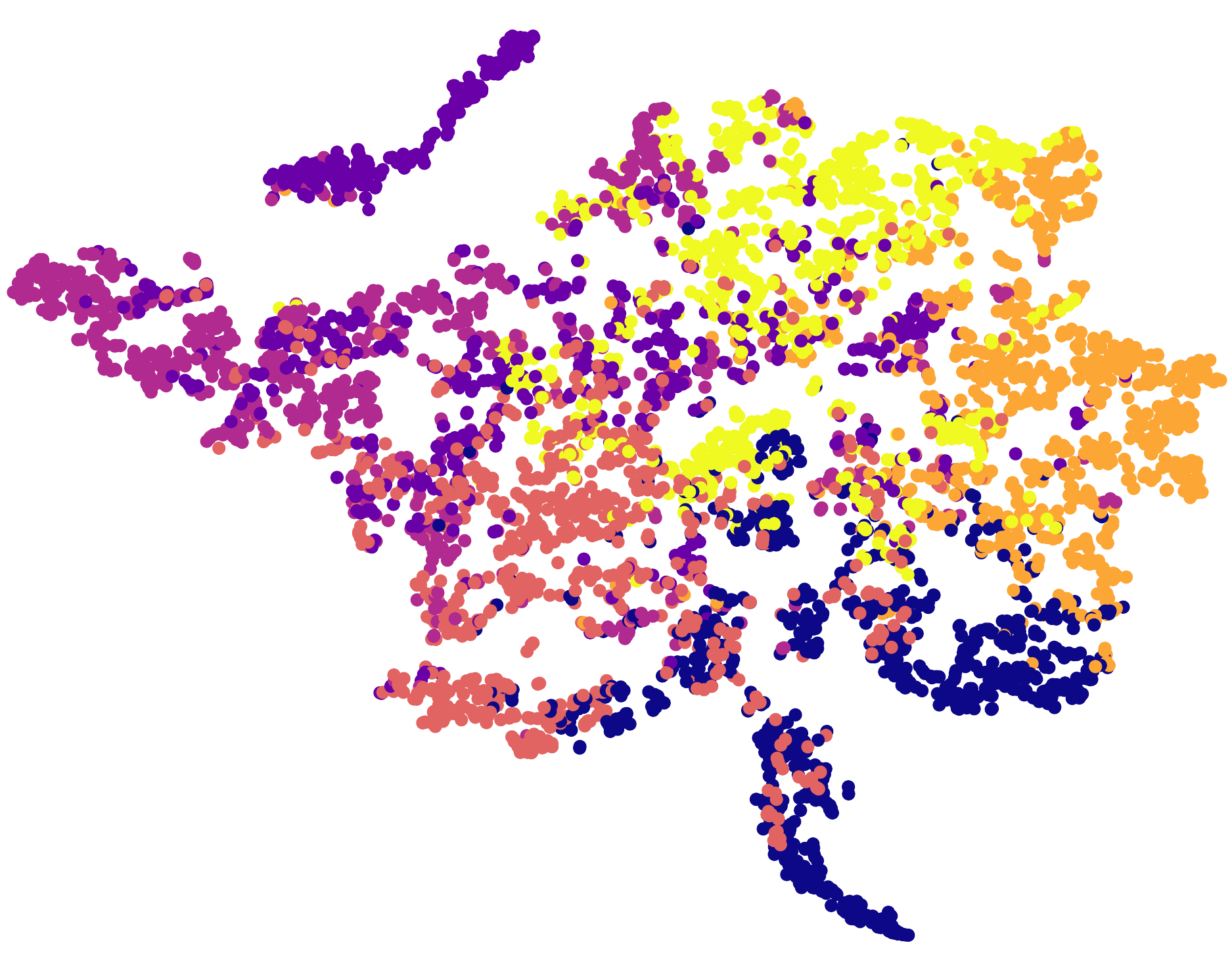}}
	\end{subfigure}\hspace{3.5mm}
	\begin{subfigure}[APPNP]{
	    \centering
		\includegraphics[width=0.13\textwidth]{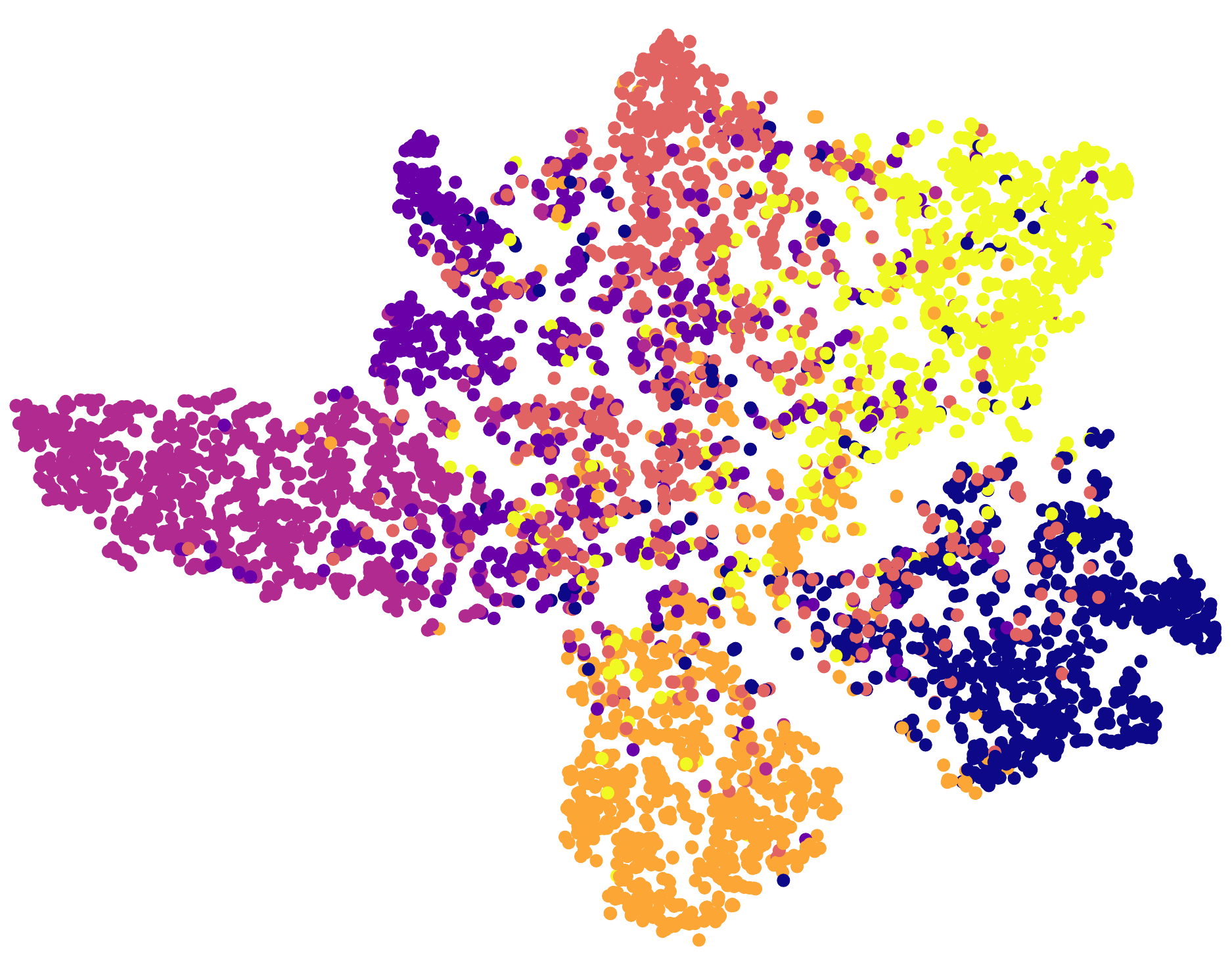}}
	\end{subfigure}\\
	\begin{subfigure}[JKNet]{
	    \centering
		\includegraphics[width=0.13\textwidth]{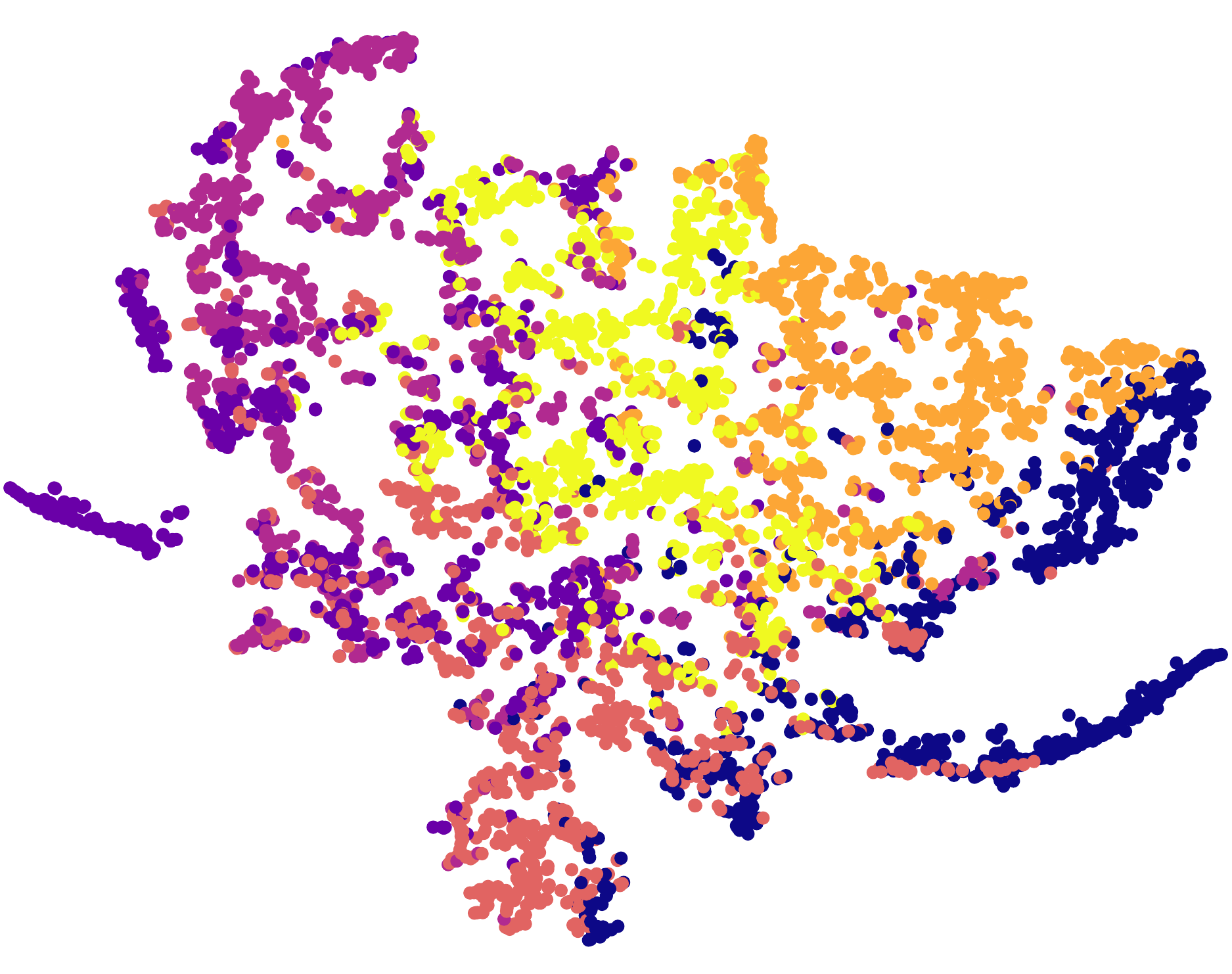}}
	\end{subfigure}\hspace{3.5mm}
	\begin{subfigure}[DAGNN]{
	    \centering
		\includegraphics[width=0.13\textwidth]{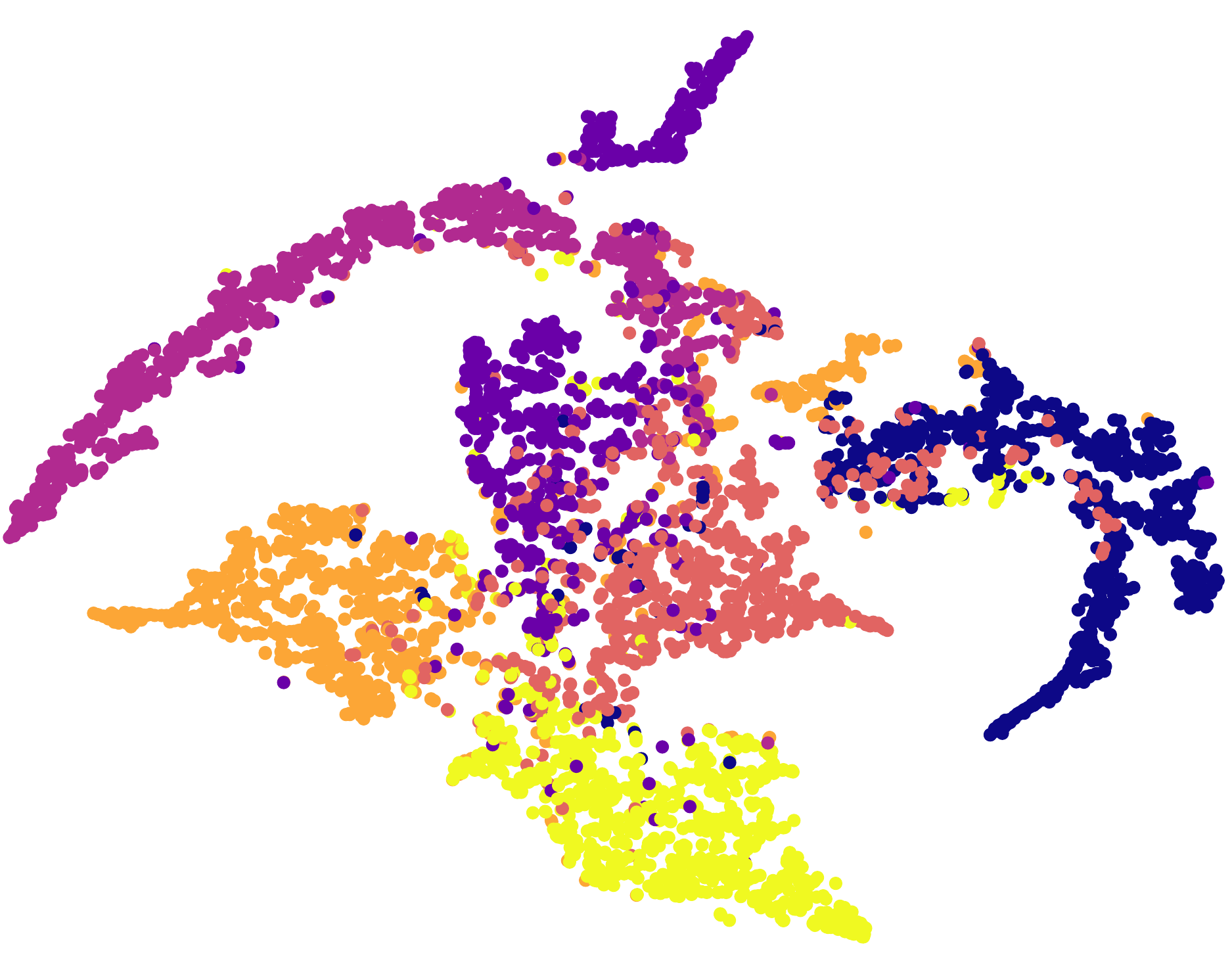}}
	\end{subfigure}\hspace{3.5mm}
	\begin{subfigure}[GNN-LF]{
	    \centering
		\includegraphics[width=0.13\textwidth]{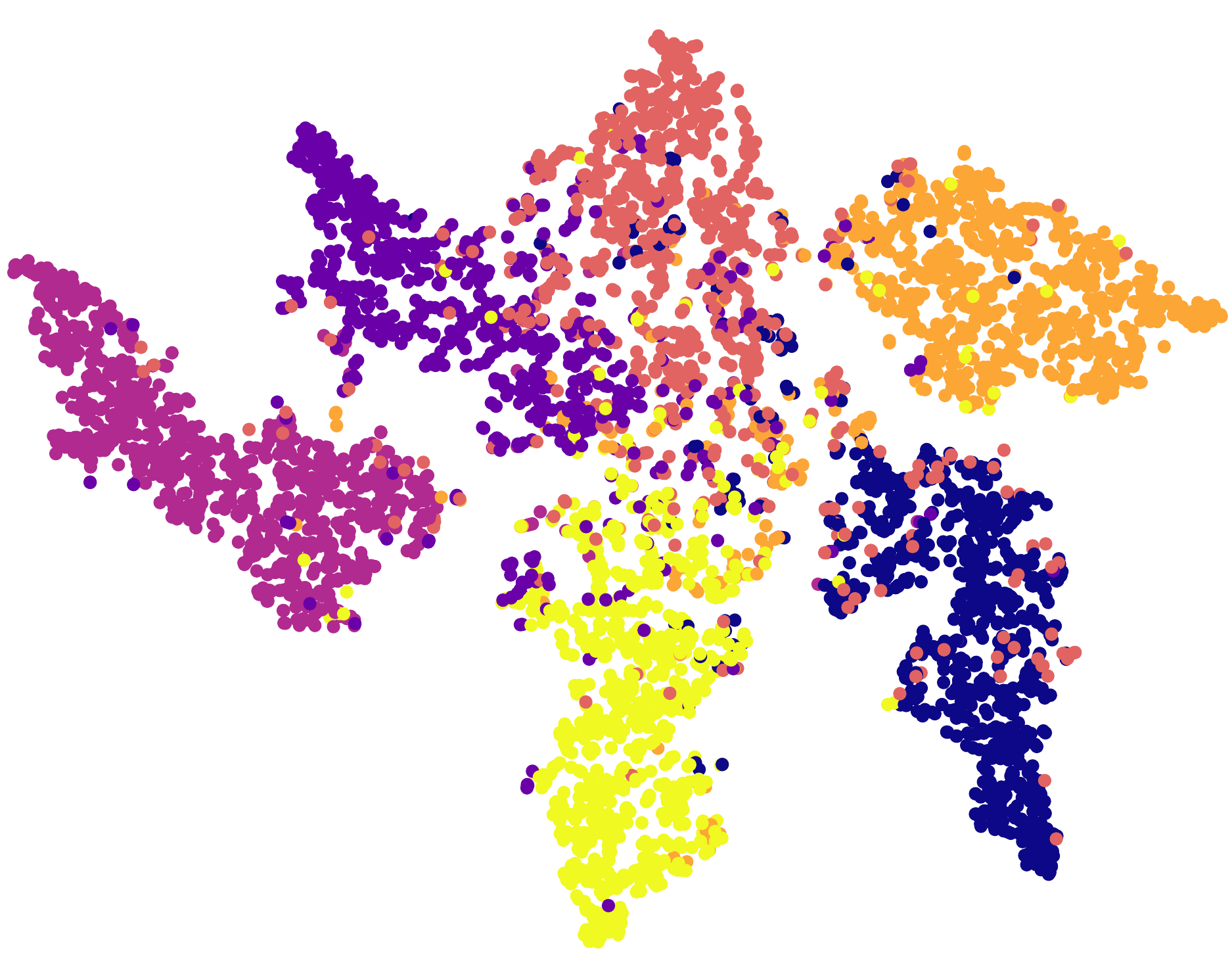}}
	\end{subfigure}\hspace{3.5mm}
	\begin{subfigure}[GNN-HF]{
	    \centering
		\includegraphics[width=0.13\textwidth]{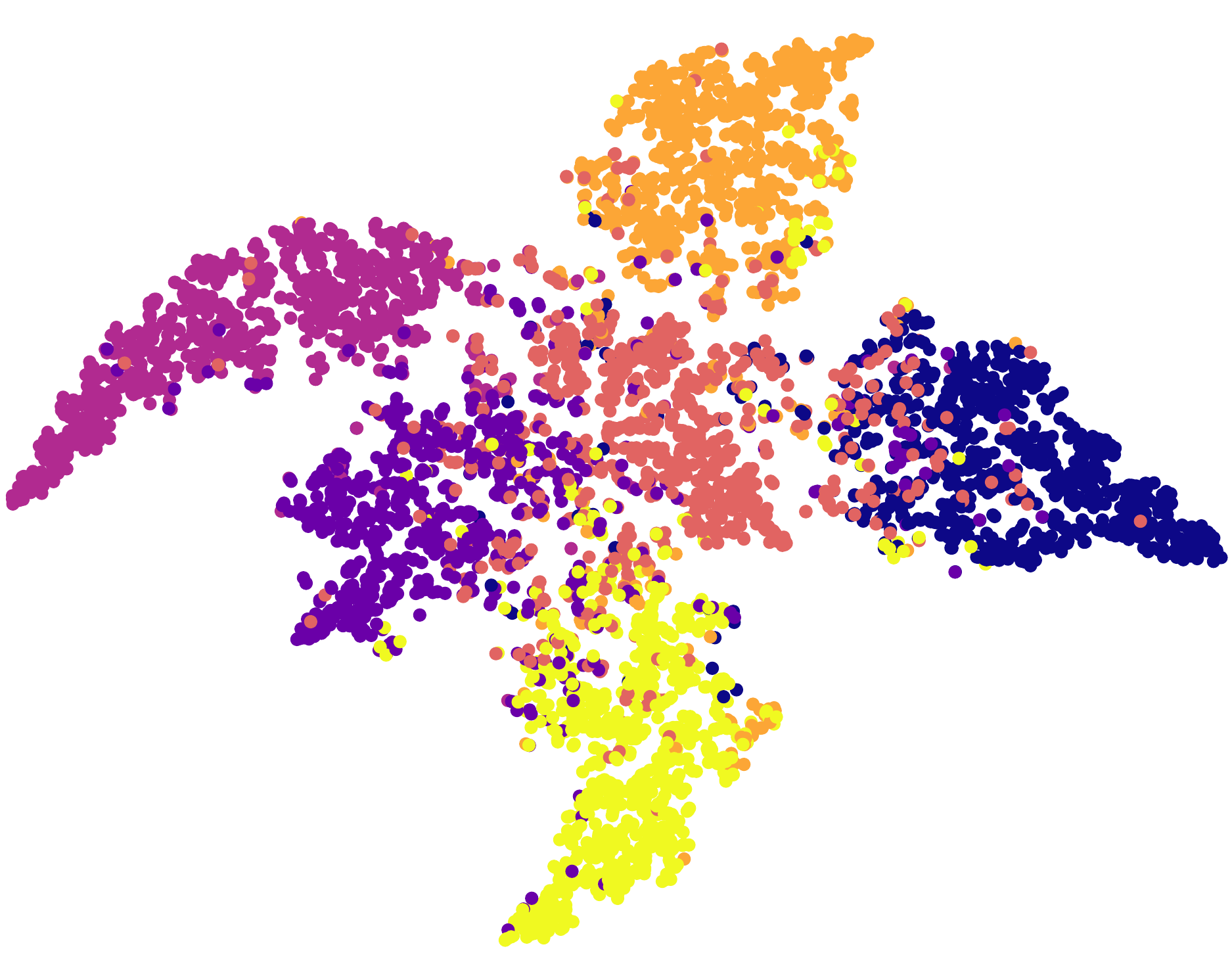}}
	\end{subfigure}\hspace{3.5mm}
	\begin{subfigure}[tsGCN (inv)]{
	    \centering
		\includegraphics[width=0.13\textwidth]{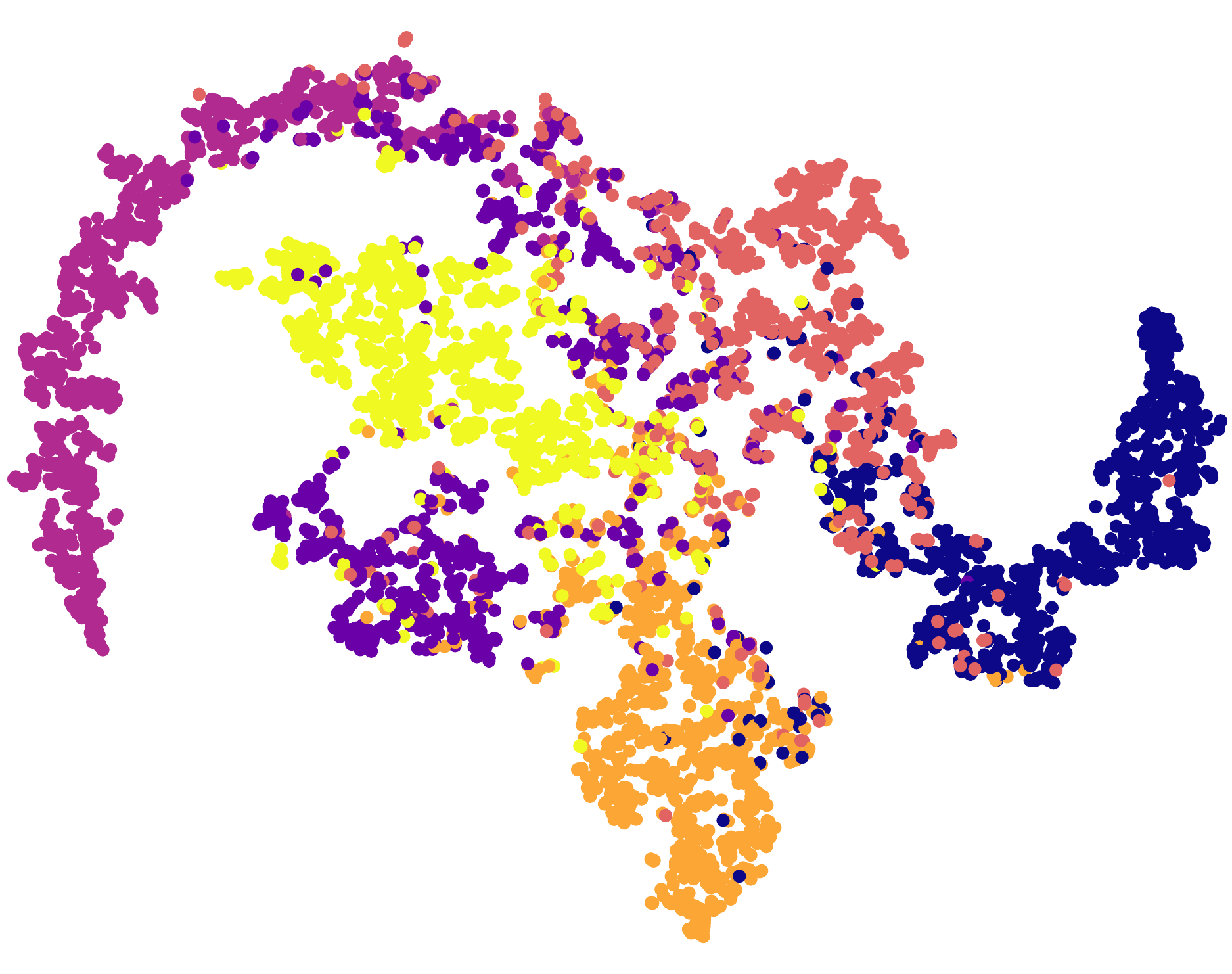}}
	\end{subfigure}\hspace{3.5mm}
	\begin{subfigure}[tsGCN]{
	    \centering
		\includegraphics[width=0.13\textwidth]{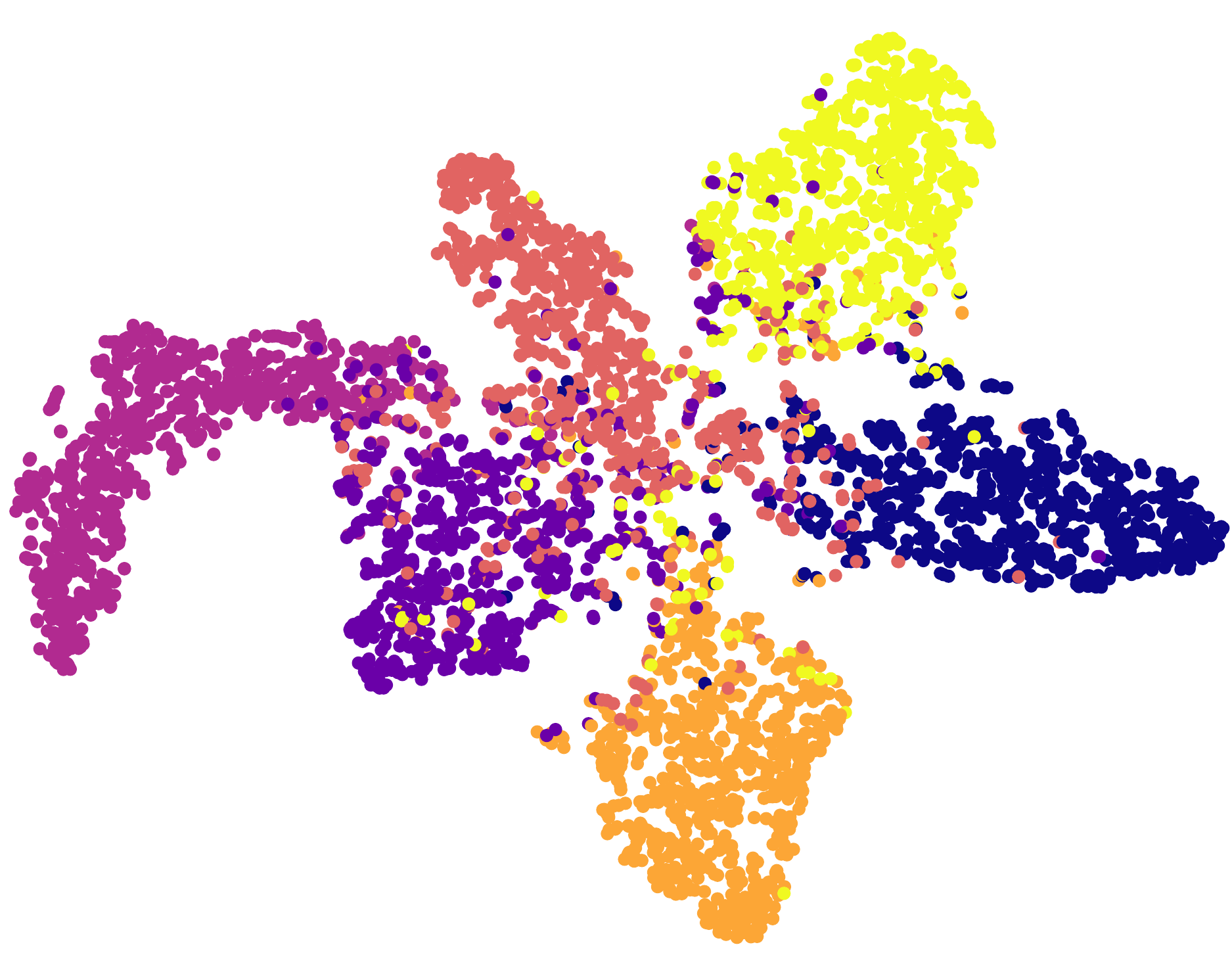}}
	\end{subfigure}
	\caption{Different methods' t-SNE visualizations on BlogCatalog, where each color corresponds to one class.}
	\label{Tsne}
\end{figure*}

\subsection{Parameter Sensitivity Analysis}

Fig.~\ref{fig:b} curves the accuracy of tsGCN w.r.t. various ranks by fixing other parameters $\alpha$ and $\beta$. 
Considering that different datasets hold different distributions, their optimal ranks to the optimization~(\ref{label_rank_approximation}) are also different. 
For example, in regard to the curves on BlogCatalog and ACM, their accuracy first go up to a high value and then keep steady, which indicates that when rank $r=\lfloor d/512 \rfloor$, the low-rank approximation is effective and efficient enough. 
When it comes to the curve on Pubmed, the trend of its performance monotonically decreases as rank $r$ becomes bigger, which implies that a very low-rank (i.e., $r=\lfloor d/2048 \rfloor$) approximation is sufficient enough to preserve abundant information for good results. 
However, with respect to the other curves such as on Flickr and Cora, the y-axis' scores generally rise to a peak first and then fall continuously as the rank $r$ increases. For these cases, the optimal ranks differ at their peaks. 

Fig.~\ref{Sensitivity} plots the accuracy of tsGCN w.r.t. ($\alpha$, $\beta$) by fixing the optimal ranks. 
On Cora, Citeseer, Pubmed, BlogCatalog, and CoraFull, tsGCN performs well with large $\alpha$ and small $\beta$; while, on ACM, Flickr, and UAI, tsGCN generates high accuracy when these two parameters are both large. 

For detailed settings of these hyperparameters, please reference the codes and datasets to be released on Github.

\subsection{Ablation Study}

The results of GCN, tsGCN-s, tsGCN-t, tsGCN (inv), and tsGCN are plotted in Fig.~\ref{Ablation_2} (notice that tsGCN-s and tsGCN-t are with semantic and topological regularizer, respectively), telling us:
\begin{itemize}
    \item The performance is unsatisfactory when the two regularizers are adopted alone, while tsGCN can always effectively fuse the two to better capture underlying structures.
    
    \item tsGCN (inv) is even worse than single-regularizer model on some datasets, indicating that the infinite-order graph convolutions implemented by the matrix inverse can pull-in instability to the model.
    
    \item Compared to GCN, tsGCN (inv) performs comparable or even worse, whereas tsGCN shows substantial improvements on all datasets, which indicates that the low-rank approximation enhances the robustness of infinite-order graph convolutions.
\end{itemize}

\subsection{Data Visualization}

Fig.~\ref{Tsne} exhibits the graph representations learned by different methods on BlogCatalog.
As can be seen clearly, the results of the three classical graph neural networks, i.e., Chebyshev, GraphSAGE and GAT, are unsatisfactory; while for the other competitors, there are:
\begin{itemize}
	\item Both tsGCN (inv) and tsGCN are better than other GCNs, which indicates that the dual-regularizer can extract more accurate inter-relationships from the topological and semantic structures.
	\item Comparing the embeddings learned by tsGCN with those of tsGCN (inv), classes in the former sub-figure are more clearly recognized and the within-clusters are more compact, which testifies the effectiveness of the low-rank approximation.
\end{itemize}

In a nutshell, the embeddings of the proposed model show the best inter-class separation and intra-class aggregation.

\section{Conclusion}

By revisiting GCN, this paper puts forward an interpretable regularizer-centered optimization framework, in which the connections between existing GCNs and diverse regularizers are revealed. It's worth mentioning that this framework provides a new perspective to interpret existing work and guide new GCNs just by designing new graph regularizers.
Impressed by the significant effectiveness of the feature based semantic graph, we further combine it with nodes' topological structures, and develop a novel dual-regularizer graph convolutional network, called tsGCN.
Since the analytical updating rule of tsGCN contains a time-consuming matrix inverse, we devise an efficient algorithm with low-rank approximation tricks.
Experiments on node classification tasks demonstrate that tsGCN performs much better than quite a few state-of-the-art competitors, and also exhibit that tsGCN runs much faster than the very recently proposed GCN variants, e.g., GNN-HF and GNN-LF.

\section{Acknowledgments}
This work is in part supported by the National Natural Science Foundation of China (Grant Nos. U21A20472 and 62276065), the Natural Science Foundation of Fujian Province (Grant No. 2020J01130193).

\section{Supplementary}


In this supplementary, we mainly present specific details to link various GCNs with various graph regularizers under the regularizer-centered optimization framework. Besides, more experimental settings and results are provided to further enrich the main paper.

\subsection{The Framework Review}
An interpretable regularizer-centered constrained optimiazation framework is induced as
\begin{align}\label{regularizer_framework}
\begin{split}
\arg\min_{\mathbf{H}^{(l)}} \mathcal{L} = \underbrace{-\text{Tr}({\mathbf{H}^{(l)}}^{\top}\mathbf{H}^{(l-1)}\mathbf{\Theta}^{(l)})}_{fitting} + \underbrace{\frac{1}{2}\mathcal{L}(\mathbf{H}^{(l)};\mathcal{G})}_{regularization} \\
\end{split}
\end{align}
\begin{displaymath}
\mbox{ s.t. } \mathbf{H}^{(l)} \in \{\mathcal{S}_{+}\; or \;\mathcal{S}\}, l\in[L-1], \mathbf{H}^{(L)} \in \mathcal{S}_{simplex},
\end{displaymath}
with the aim to unify various GCNs in an interpretable way, and also to guide the design of new GCN variants.
Note that the first term in optimization (\ref{regularizer_framework}) is equivalent to the fitting loss between the forward  propagation $\mathbf{H}^{(l-1)}\mathbf{\Theta}^{(l)}$ and the output $\mathbf{H}^{(l)}$, while the second term is the priors-based graph regularizer. Besides, $\mathcal{S}$, $\mathcal{S}_{+}$ and $\mathcal{S}_{simplex}$ are separately defined to be
\begin{equation}
\mathcal{S} = \{\mathbf{s}\in\mathbb{R}^{d}\},
\end{equation} 
\begin{equation}
\mathcal{S}_{+} = \{\mathbf{s}\in\mathbb{R}^{d} | \mathbf{s} \geq \mathbf{0}\},
\end{equation}
and
\begin{equation}
\mathcal{S}_{simplex} = \{\mathbf{s}\in\mathbb{R}^{d} | \mathbf{1}^{\top}\mathbf{s} = 1, \mathbf{s} \geq \mathbf{0}\}.
\end{equation}
The above three sets correspond to the $\text{Identity}(\cdot)$, $\text{Relu}(\cdot)$, and $\text{Softmax}(\cdot)$ activation functions frequently used in graph convolutional networks, respectively.

It's claimed that by designing different regularizers, this framework can give birth to different GCN methods. In the following, we will give specific details about how they could be derived from optimization (\ref{regularizer_framework}).

\subsection{Link Various GCNs with Various Regularizers}
\begin{table*}[!htbp]
    \centering
    \resizebox{0.99\textwidth}{!}{
    \begin{tabular}{|m{1.5cm}<{\centering}|m{7.1cm}<{\centering}|m{5.9cm}<{\centering}|m{3.6cm}|}
    \toprule		
     \multicolumn{1}{|c|}{Methods} &  \multicolumn{1}{c|}{Propagation Rules} & \multicolumn{1}{c|}{Regularizer $\mathcal{L}(\mathbf{H}^{(l)};\mathcal{G})$} & \multicolumn{1}{c|}{Projective Set} \\
    \midrule
    GCN   & $\mathbf{H}^{(l)} = \sigma\left(\widehat{\mathbf{A}}\mathbf{H}^{(l-1)}\mathbf{\Theta}^{(l)}\right)$  & $ \text{Tr}\left({\mathbf{H}^{(l)}}^{\top}\widetilde{\mathbf{L}}\mathbf{H}^{(l)}\right)$ & $\left\{\begin{matrix}\mathcal{S}^{(l)} = \mathcal{S}_{+}, l\in[L$$-$$1],\\ \mathcal{S}^{(L)} =\mathcal{S}_{simplex}\end{matrix}\right.$\\
    SGC  & $\mathbf{H}^{(l)} = \sigma\left({\widehat{\mathbf{A}}}\mathbf{H}^{(l-1)}\mathbf{\Theta}^{(l)}\right)$ &  $\text{Tr}\left({\mathbf{H}^{(l)}}^{\top}\widetilde{\mathbf{L}}\mathbf{H}^{(l)}\right)$ & $\left\{\begin{matrix}\mathcal{S}^{(l)} = \mathcal{S}, l\in[L$$-$$1],\\ \mathcal{S}^{(L)} =\mathcal{S}_{simplex}\end{matrix}\right.$\\
    APPNP & $\mathbf{H}^{(l)} = \sigma\left((1-\alpha)\widehat{\mathbf{A}}\mathbf{H}^{(l-1)}+\alpha\mathbf{H}^{(0)}\right)$ & $ \text{Tr}\left(\frac{1}{1-\alpha}{\mathbf{H}^{(l)}}^{\top}{\widehat{\mathbf{A}}}^{-1}(\mathbf{H}^{(l)}-2\alpha\mathbf{H}^{(0)})\right)$ & $\left\{\begin{matrix}\mathcal{S}^{(l)} = \mathcal{S}, l\in[L$$-$$1],\\ \mathcal{S}^{(L)} =\mathcal{S}_{simplex}\end{matrix}\right.$\\
    JKNet & $\mathbf{H}^{(l)} = \sigma\left(\sum_{k=1}^{K}\alpha_{k}{\widehat{\mathbf{A}}}^{k}\mathbf{H}^{(l-1)}\mathbf{\Theta}^{(l)}\right)$ & $\text{Tr}\left({\mathbf{H}^{(l)}}^{\top}{\widehat{\mathbf{A}}}^{-1}(\mathbf{I} + \beta\widetilde{\mathbf{L}})\mathbf{H}^{(l)}\right)$   & $\left\{\begin{matrix}\mathcal{S}^{(l)} = \mathcal{S}, l\in[L$$-$$1],\\ \mathcal{S}^{(L)} =\mathcal{S}_{simplex}\end{matrix}\right.$\\
    DAGNN  & $\mathbf{H}^{(L)} =  \sigma\left(\sum_{k=0}^{K}\alpha_{k}{\widehat{\mathbf{A}}}^{k}\mathbf{H}^{(0)}\right)$ & $\text{Tr}\left({\mathbf{H}^{(l)}}^{\top}(\mathbf{I} + \beta\widetilde{\mathbf{L}})\mathbf{H}^{(l)}\right)$   & $\left\{\begin{matrix}\mathcal{S}^{(l)} = \mathcal{S}, l\in[L$$-$$1],\\ \mathcal{S}^{(L)} =\mathcal{S}_{simplex}\end{matrix}\right.$ \\
    GNN-HF &  $\mathbf{H}^{(l)} = \sigma\left((\mathbf{I} + \alpha\widehat{\mathbf{L}})^{-1}(\mathbf{I} + \beta\widehat{\mathbf{L}})\mathbf{H}^{(l-1)}\mathbf{\Theta}^{(l)}\right)$ & $\text{Tr}\left({\mathbf{H}^{(l)}}^{\top}(\mathbf{I} + \beta\widehat{\mathbf{L}})^{-1}(\mathbf{I} + \alpha\widehat{\mathbf{L}})\mathbf{H}^{(l)}\right)$ & $\left\{\begin{matrix}\mathcal{S}^{(l)} = \mathcal{S}_{+}, l\in[L$$-$$1],\\ \mathcal{S}^{(L)} =\mathcal{S}_{simplex}.\end{matrix}\right.$\\
    GNN-LF &  $\mathbf{H}^{(l)} = \sigma\left((\mathbf{I} + \alpha\widehat{\mathbf{A}})^{-1}(\mathbf{I} + \beta\widehat{\mathbf{A}})\mathbf{H}^{(l-1)}\mathbf{\Theta}^{(l)}\right)$ & $\text{Tr}\left({\mathbf{H}^{(l)}}^{\top}(\mathbf{I} + \beta\widehat{\mathbf{A}})^{-1}(\mathbf{I} + \alpha\widehat{\mathbf{A}})\mathbf{H}^{(l)} \right)$ & $\left\{\begin{matrix}\mathcal{S}^{(l)} = \mathcal{S}_{+}, l\in[L$$-$$1],\\ \mathcal{S}^{(L)} =\mathcal{S}_{simplex}\end{matrix}\right.$\\
    \midrule
    tsGCN & $\mathbf{H}^{(l)} = \sigma\left((\mathbf{I} +  \alpha\widetilde{\mathbf{L}}_{\mathcal{G}} + \beta\widetilde{\mathbf{L}}_{\mathcal{X}})^{-1}\mathbf{H}^{(l-1)}\mathbf{\Theta}^{(l)}\right)$ & $\text{Tr}\left({\mathbf{H}^{(l)}}^{\top}(\mathbf{I} + \alpha\widetilde{\mathbf{L}}_{\mathcal{G}} + \beta\widetilde{\mathbf{L}}_{\mathcal{X}})\mathbf{H}^{(l)}\right)$ &  $\left\{\begin{matrix}\mathcal{S}^{(l)} = \mathcal{S}_{+}, l\in[L$$-$$1],\\ \mathcal{S}^{(L)} =\mathcal{S}_{simplex}\end{matrix}\right.$\\
    \bottomrule
    \end{tabular}}
    \caption{Different regularizers can derive different GCN variants under the regularizer-centered optimization framework.}
    \label{ConnectionRegularizerGCN}
\end{table*}

\begin{theorem}
The updating rule of the vanilla GCN \cite{KipfWelling2017SemiSupervised}
\begin{align}\label{GCNupdatingrule}
\mathbf{H}^{(l)} = \sigma\left(\widehat{\mathbf{A}}\mathbf{H}^{(l-1)}\mathbf{\Theta}^{(l)}\right), l \in [L],
\end{align}
is equivalent to solving the following optimization
\begin{equation}
\mathbf{H}^{(l)} = \arg\min_{\mathbf{H}\in\mathcal{S}^{(l)}}\mathcal{J}^{(l)}
\end{equation}
\begin{displaymath}
s.t.\; \mathcal{S}^{(l)}\in\{\mathcal{S}\; or \;\mathcal{S}_{+}\; or\; \mathcal{S}_{simplex}\},
\end{displaymath}
where  
\begin{align}\label{GCNoptimization}
\mathcal{J}^{(l)} = -{\rm{Tr}}\left(\mathbf{H}^{\top}\mathbf{H}^{(l-1)}\mathbf{\Theta}^{(l)}\right) + \frac{1}{2}{\rm{Tr}}\left(\mathbf{H}^{\top}\widetilde{\mathbf{L}}\mathbf{H}\right).
\end{align}
\end{theorem}
\begin{proof}
Taking derivative of $\mathcal{J}^{(l)}$ w.r.t. $\mathbf{H}$, we obtain
\begin{align}
\frac{\partial \mathcal{J}^{(l)}}{\partial \mathbf{H}} =  - \mathbf{H}^{(l-1)}\mathbf{\Theta}^{(l)} + \widetilde{\mathbf{L}}\mathbf{H};
\end{align}
if it ($\frac{\partial \mathcal{J}^{(l)}}{\partial \mathbf{H}}$) is further set to zero, then there yields
\begin{equation}
\mathbf{H}^{*} = (\mathbf{I}-\widetilde{\mathbf{A}})^{-1}\mathbf{H}^{(l-1)}\mathbf{\Theta}^{(l)}.
\end{equation}

By projecting $\mathbf{H}^{*}$ onto $\mathcal{S}^{(l)}$, we could arrive at
\begin{equation}\label{eq:updating_with_inverse}
\mathbf{H}^{(l)}=\sigma(\mathbf{H}^{*}).
\end{equation}
Notably, $(\mathbf{I}-\widetilde{\mathbf{A}})^{-1} = \sum_{i=0}^{\infty}\widetilde{\mathbf{A}}^{i}$; 
and when its first-order approximation is leveraged, i.e., $(\mathbf{I}-\widetilde{\mathbf{A}})^{-1} \approx \mathbf{I} + \tilde{\mathbf{A}} = \widehat{\mathbf{A}}$, Eq.~\eqref{eq:updating_with_inverse} gives birth to the updating rule \eqref{GCNupdatingrule}. 

The above analyses reveal that when the regularizer is designed to $\frac{1}{2}{\rm{Tr}}\left(\mathbf{H}^{\top}\widetilde{\mathbf{L}}\mathbf{H}\right)$, the framework \eqref{regularizer_framework} could generate the vanilla GCN \cite{KipfWelling2017SemiSupervised}.
\end{proof}

\begin{theorem}
	Given $\mathbf{H}^{(0)} = f_{\Theta}^{MLP}(\mathbf{X})$ and $\alpha \in [0, 1)$, the updating rule of APPNP \cite{KlicperaBojchevski2019Predict}
	\begin{align}
	\mathbf{H}^{(l)} = \sigma\left((1-\alpha)\widehat{\mathbf{A}}\mathbf{H}^{(l-1)}+\alpha\mathbf{H}^{(0)}\right),\; l \in [L], 
	\end{align}
	is equivalent to solving the following optimization
	\begin{equation}
	\mathbf{H}^{(l)} = \arg\min_{\mathbf{H}\in\mathcal{S}^{(l)}}\mathcal{J}^{(l)}, 
	\end{equation}
	\begin{displaymath}
	s.t. ~\mathbf{H}^{(l)} = \mathcal{S},\; l\in[L-1],\; \mathbf{H}^{(L)} = \mathcal{S}_{simplex},
	\end{displaymath}
	where 
	\begin{equation}
	\begin{aligned}
	\mathcal{J}^{(l)} & =  -{\rm{Tr}}\left(\mathbf{H}^{\top}\mathbf{H}^{(l-1)}\mathbf{\Theta}^{(l)}\right)\\ &+  \frac{1}{2}{\rm{Tr}}\left(\frac{1}{1-\alpha}\mathbf{H}^{\top}{\widehat{\mathbf{A}}}^{-1}(\mathbf{H}-2\alpha\mathbf{H}^{(0)})\right).
	\end{aligned}
	\end{equation}
\end{theorem}
\begin{proof}
	Taking the derivative of $\mathcal{J}^{(l)}$ w.r.t. $\mathbf{H}$ and setting it to zero, we come to 
	\begin{align}
	\frac{\partial \mathcal{J}^{(l)}}{\partial \mathbf{H}} =  - \mathbf{H}^{(l-1)} + \frac{1}{1-\alpha}{\widehat{\mathbf{A}}}^{-1}(\mathbf{H}-\alpha\mathbf{H}^{(0)}) = \mathbf{0}, 
	\end{align}
	which leads to 
	\begin{align}
	\mathbf{H}^{*} = (1-\alpha)\widehat{\mathbf{A}}\mathbf{H}^{(l-1)}+\alpha\mathbf{H}^{(0)}.
	\end{align}
	
	By projecting $\mathbf{H}^{*}$ onto $\mathcal{S}^{(l)}$, we could achieve 
	\begin{align}
	\mathbf{H}^{(l)} = \sigma\left((1-\alpha)\widehat{\mathbf{A}}\mathbf{H}^{(l-1)}+\alpha\mathbf{H}^{(0)}\right),\; l\in[L], 
	\end{align}
	which completes the proof.
	
	The above analyses reveal that when the regularizer is devised to $\frac{1}{2}{\rm{Tr}}\left(\frac{1}{1-\alpha}\mathbf{H}^{\top}{\widehat{\mathbf{A}}}^{-1}(\mathbf{H}-2\alpha\mathbf{H}^{(0)})\right)$, the framework \eqref{regularizer_framework} could give birth to APPNP \cite{KlicperaBojchevski2019Predict}.
\end{proof}

\begin{theorem}\label{JKNetproof}
	The updating rule of JKNet \cite{XuLiTian2018Representation} 
	\begin{align}
	\begin{split}
	\mathbf{H}^{(l)} = \sigma\left(\sum_{k=1}^{K}\alpha_{k}{\widehat{\mathbf{A}}}^{k}\mathbf{H}^{(l-1)}\mathbf{\Theta}^{(l)}\right),\; l\in [L],
	\end{split}
	\end{align}
	is equivalent to solving the following optimization 
	\begin{equation}
	\mathbf{H}^{(l)} = \arg\min_{\mathbf{H}\in\mathcal{S}^{(l)}}\mathcal{J}^{(l)}
	\end{equation}
	\begin{displaymath}
	s.t. ~\mathbf{H}^{(l)} = \mathcal{S}, l\in[L-1], \mathbf{H}^{(L)} = \mathcal{S}_{simplex},
	\end{displaymath}
	where
	\begin{equation}
	\begin{aligned}
	\mathcal{J}^{(l)} &=-{\rm{Tr}}\left(\mathbf{H}^{\top}\mathbf{H}^{(l-1)}\mathbf{\Theta}^{(l)}\right) \\ &+ \frac{1}{2}{\rm{Tr}}\left(\mathbf{H}^{\top}{\widehat{\mathbf{A}}}^{-1}(\mathbf{I} + \beta\widetilde{\mathbf{L}})\mathbf{H}\right).
	\end{aligned}
	\end{equation}
\end{theorem}
\begin{proof}
	Taking the derivative of $\mathcal{J}^{(l)}$ w.r.t. $\mathbf{H}$ and setting it to zero, we have 
	\begin{align}
	\frac{\partial \mathcal{J}^{(l)}}{\partial \mathbf{H}} =  - \mathbf{H}^{(l-1)}\mathbf{\Theta}^{(l)} + {\widehat{\mathbf{A}}}^{-1}(\mathbf{I} + \beta\widetilde{\mathbf{L}})\mathbf{H} = \mathbf{0}, 
	\end{align}
	which leads to
	\begin{equation}
	\mathbf{H}^{*}=\frac{1}{\beta+1}\left(\mathbf{I} - \frac{\beta}{\beta+1}\widehat{\mathbf{A}}\right)^{-1}\widehat{\mathbf{A}}\mathbf{H}^{(l-1)}\mathbf{\Theta}^{(l)}. 
	\end{equation}
	
	It is noted that the spectral radius of $\frac{\beta}{\beta+1}\widehat{\mathbf{A}}$ is smaller than one, indicating 
	\begin{align}
	\left(\mathbf{I} - \frac{\beta}{\beta+1}\widehat{\mathbf{A}}\right)^{-1} = \sum_{k=0}^{\infty}\left[\frac{\beta}{\beta+1}\widehat{\mathbf{A}}\right]^{k}. 
	\end{align}
	
	If its $(K$$-$$1)$-order approximation is employed, then there goes 
	\begin{align}
	\left(\mathbf{I} - \frac{\beta}{\beta+1}\widehat{\mathbf{A}}\right)^{-1} \approx \sum_{k=0}^{K-1}\frac{\beta^{k}}{(\beta+1)^{k}}\widehat{\mathbf{A}}^{k}, 
	\end{align}
	which suggests that $\mathbf{H}^{*}$ can be approximated by 
	\begin{align}
	\mathbf{H}^{*} = \sum_{k=1}^{K}\frac{\beta^{k-1}}{(\beta+1)^{k}}\widehat{\mathbf{A}}^{k}\mathbf{H}^{(l-1)}\mathbf{\Theta}^{(l)}.
	\end{align}
	
	If denote $\alpha_{k} = \frac{\beta^{k-1}}{(\beta+1)^{k}}$ ($k\in[K]$), then $\{\alpha_{k}\}_{k=1}^{\infty}$ is a set of parameters with $\sum_{k=1}^{\infty} \alpha_{k} = \frac{1}{\beta+1} \frac{1}{1-\frac{\beta}{\beta+1}} = 1$. 
	
	By projecting $\mathbf{H}^{*}$ onto $\mathcal{S}^{(l)}$, we can realize 
	\begin{align}
	\mathbf{H}^{(l)} = \sigma\left(\sum_{k=1}^{K}\alpha_{k}{\widehat{\mathbf{A}}}^{k}\mathbf{H}^{(l-1)}\mathbf{\Theta}^{(l)}\right), l \in [L],      
	\end{align}
	which completes the proof. 
	
	The above analyses reveal that when the regularizer is devised to $\frac{1}{2}{\rm{Tr}}\left(\mathbf{H}^{\top}{\widehat{\mathbf{A}}}^{-1}(\mathbf{I} + \beta\widetilde{\mathbf{L}})\mathbf{H}\right)$, the framework \eqref{regularizer_framework} can produce JKNet \cite{XuLiTian2018Representation}.
\end{proof}

\begin{theorem}
	Given $\mathbf{H}^{(0)} = f_{\Theta}^{MLP}(\mathbf{X})$ and a trainable
	projection vector $\alpha \in \mathbb{R}^{K+1}$, the updating rule of DAGNN \cite{LiuGaoJi2020Towards}
	\begin{align}
	\mathbf{H}^{(l)} =  \sigma\left(\sum_{k=0}^{K}\alpha_{k}{\widehat{\mathbf{A}}}^{k}\mathbf{H}^{(0)}\right), 
	\end{align}
	is equivalent to solving the following optimization 
	\begin{equation}
	\mathbf{H}^{(l)} = \arg\min_{\mathbf{H}\in\mathcal{S}^{(l)}}\mathcal{J}^{(l)}
	\end{equation}
	\begin{displaymath}
	s.t. ~\mathbf{H}^{(l)} = \mathcal{S},\; l\in[L-1],\; \mathbf{H}^{(L)} = \mathcal{S}_{simplex},
	\end{displaymath}
	where
	\begin{equation}
	\begin{aligned}
	\mathcal{J}^{(l)} &= -{\rm{Tr}}\left(\mathbf{H}^{\top}\mathbf{H}^{(l-1)}\mathbf{\Theta}^{(l)}\right) \\ &+ \frac{1}{2}{\rm{Tr}}\left(\mathbf{H}^{\top}(\mathbf{I} + \beta\widetilde{\mathbf{L}})\mathbf{H}\right).
	\end{aligned}
	\end{equation}
\end{theorem}
\begin{proof}
	Taking the derivative of $\mathcal{J}^{(l)}$ w.r.t. $\mathbf{H}$ and setting it to zero, we can harvest
	\begin{align}
	\frac{\partial \mathcal{J}^{(l)}}{\partial \mathbf{H}} =  - \mathbf{H}^{(0)} + (\mathbf{I} + \beta\widetilde{\mathbf{L}})\mathbf{H} = \mathbf{0}.  
	\end{align}
	
	Similar to the proof of Theorem \ref{JKNetproof}, the $K$-order approximation of $\left(\mathbf{I} + \beta\widetilde{\mathbf{L}}\right)^{-1} =\frac{1}{\beta+1} \left(\mathbf{I} - \frac{\beta}{\beta+1}\widehat{\mathbf{A}}\right)^{-1}$ is utilized, and then we obtain 
	\begin{align}
	\mathbf{H}^{*} =  \sum_{k=0}^{K}\alpha_{k}{\widehat{\mathbf{A}}}^{k}\mathbf{H}^{(0)}. 
	\end{align}
	
	By projecting $\mathbf{H}^{*}$ onto $\mathcal{S}^{(l)}$, we can arrive at 
	\begin{align}
	\mathbf{H}^{(l)} = \sigma\left(\sum_{k=0}^{K}\alpha_{k}{\widehat{\mathbf{A}}}^{k}\mathbf{H}^{(0)}\right),\; l \in [L],      
	\end{align}
	which completes the proof.
	
	The above analyses reveal that when the regularizer is devised to $\frac{1}{2}{\rm{Tr}}\left(\mathbf{H}^{\top}(\mathbf{I} + \beta\widetilde{\mathbf{L}})\mathbf{H}\right)$, the framework \eqref{regularizer_framework} can produce DAGNN \cite{LiuGaoJi2020Towards}. 
\end{proof}

\begin{theorem}
	The updating rule of GNN-HF \cite{ZhuWang2021Interpreting} 
	\begin{equation}\label{GNNHFclosed}
	\begin{split}
	\mathbf{H}^{(l)} &= \sigma((\beta+\frac{1}{\alpha})\mathbf{I} \\&+ (1-\beta-\frac{1}{\alpha})\widehat{\mathbf{A}}^{-1}(\mathbf{I} + \beta\widehat{\mathbf{L}})\mathbf{H}^{(l-1)}\mathbf{\Theta}^{(l)}
	),
	\end{split}
	\end{equation}
	is equivalent to solving the following optimization
	\begin{equation}
	\mathbf{H}^{(l)} = \arg\min_{\mathbf{H}\in\mathcal{S}^{(l)}}\mathcal{J}^{(l)}
	\end{equation}
	\begin{displaymath}
	s.t. ~\mathbf{H}^{(l)} = \mathcal{S}_{+},\; l\in[L-1],\; \mathbf{H}^{(L)} = \mathcal{S}_{simplex},
	\end{displaymath}
	where
	\begin{equation}  
	\begin{aligned}
	\mathcal{J}^{(l)} &=-{\rm{Tr}}\left(\mathbf{H}^{\top}\mathbf{H}^{(l-1)}\mathbf{\Theta}^{(l)}\right) \\&+ \frac{1}{2}{\rm{Tr}}\left(\mathbf{H}^{\top}(\mathbf{I} + \mu\widehat{\mathbf{L}})^{-1}(\mathbf{I} + \lambda\widehat{\mathbf{L}})\mathbf{H}\right)
	\end{aligned}
	\end{equation}
	with $\lambda = \beta + \frac{1}{\alpha} - 1$ and $\mu = \beta$. 
\end{theorem}
\begin{proof}
	Taking the derivative of $\mathcal{J}^{(l)}$ w.r.t. $\mathbf{H}$ and setting it to zero, we own
	\begin{equation}
	\frac{\partial \mathcal{J}^{(l)}}{\partial \mathbf{H}} =  - \mathbf{H}^{(l-1)}\mathbf{\Theta}^{(l)} + (\mathbf{I} + \mu\widehat{\mathbf{L}})^{-1}(\mathbf{I} + \lambda\widehat{\mathbf{L}})\mathbf{H} = \mathbf{0}. 
	\end{equation}
	which yields
	\begin{align}\label{Closd}
	\mathbf{H}^{*} = (\mathbf{I} + \lambda\widehat{\mathbf{L}})^{-1}(\mathbf{I} + \mu\widehat{\mathbf{L}})\mathbf{H}^{(l-1)}\mathbf{\Theta}^{(l)}.       
	\end{align}
	
	Substituting $\lambda = \beta + \frac{1}{\alpha} - 1$ and $\mu = \beta$ into Eq.~\eqref{Closd}, we obtain
	\begin{align}
	\begin{split}
	(\mathbf{I} + \lambda\widehat{\mathbf{L}})^{-1} &=  \left((1 + \lambda)\mathbf{I} - \lambda\widehat{\mathbf{A}}\right)^{-1} \\
	&= \left((\beta + \frac{1}{\alpha})\mathbf{I} + (1-\beta-\frac{1}{\alpha})\widehat{\mathbf{A}}\right)^{-1}. 
	\end{split}
	\end{align}
	
	By projecting $\mathbf{H}^{*}$ onto $\mathcal{S}^{(l)}$, we can ahieve 
	\begin{align}
	\mathbf{H}^{(l)} = \sigma\left(\mathbf{H}^{*}\right),\; l \in [L],     
	\end{align}
	which completes the proof. 
	
	The above analyses reveal that when the regularizer is devised to $\frac{1}{2}{\rm{Tr}}\left(\mathbf{H}^{\top}(\mathbf{I} + \mu\widehat{\mathbf{L}})^{-1}(\mathbf{I} + \lambda\widehat{\mathbf{L}})\mathbf{H}\right)$, the framework \eqref{regularizer_framework} can produce GNN-HF \cite{ZhuWang2021Interpreting}.
\end{proof}

\begin{theorem}
	The updating rule of GNN-LF \cite{ZhuWang2021Interpreting} 
	\begin{equation}\label{GNNLFclosed}
	\begin{split}
	\mathbf{H}^{(l)} &= \sigma((\beta\mathbf{I}+(1-\beta)\widehat{\mathbf{A}} \\&+ (\frac{1}{\alpha}-1)\widehat{\mathbf{L}})^{-1}(\beta\mathbf{I} + (1-\beta)\widehat{\mathbf{A}})\mathbf{H}^{(l-1)}), 
	\end{split}
	\end{equation}
	is equivalent to solving the following optimization 
	\begin{equation}
	\mathbf{H}^{(l)} = \arg\min_{\mathbf{H}\in\mathcal{S}^{(l)}}\mathcal{J}^{(l)}
	\end{equation}
	\begin{displaymath}
	s.t. ~\mathbf{H}^{(l)} = \mathcal{S}_{+},\; l\in[L-1],\; \mathbf{H}^{(L)} = \mathcal{S}_{simplex},
	\end{displaymath}
	where
	\begin{equation}
	\begin{aligned}
	J^{(l)} &=-Tr\left(\mathbf{H}^{\top}\mathbf{H}^{(l-1)}\mathbf{\Theta}^{(l)}\right) \\&+ \frac{1}{2}Tr\left(\mathbf{H}^{\top}(\mathbf{I} + \mu\widehat{\mathbf{A}})^{-1}(\mathbf{I} + \lambda\widehat{\mathbf{A}})\mathbf{H}\right)
	\end{aligned}
	\end{equation}
	with $\lambda = \frac{-\alpha\beta+2\alpha-1}{\alpha\beta-\alpha+1}$ and $\mu = \frac{1}{\beta}-1$. 
\end{theorem}
\begin{proof}
	Taking the derivative of $\mathcal{J}^{(l)}$ w.r.t. $\mathbf{H}$ and setting it to zero, we can get
	\begin{equation}
	\frac{\partial \mathcal{J}^{(l)}}{\partial \mathbf{H}} =  - \mathbf{H}^{(l-1)}\mathbf{\Theta}^{(l)} + (\mathbf{I} + \mu\widehat{\mathbf{A}})^{-1}(\mathbf{I} + \lambda\widehat{\mathbf{A}})\mathbf{H} = \mathbf{0},  
	\end{equation} 
	which leads to 
	\begin{align}\label{LFClosd}
	\mathbf{H}^{*} = (\mathbf{I} + \lambda\widehat{\mathbf{A}})^{-1}(\mathbf{I} + \mu\widehat{\mathbf{A}})\mathbf{H}^{(l-1)}\mathbf{\Theta}^{(l)}.       
	\end{align}
	\begin{table*}[!tbp]
	\caption{Specific ($\alpha$, $\beta$, $r$) and other parameters of tsGCN on all datasets.}
	\resizebox{0.99\textwidth}{!}{
		\begin{tabular}{m{3cm}<{\centering}|m{1.3cm}<{\centering}m{1.3cm}<{\centering}m{1.5cm}<{\centering}m{2.2cm}<{\centering}m{2.2cm}<{\centering}m{2.2cm}<{\centering}}
			\toprule[1pt]
			Datasets/Parameters & $\alpha$  & $\beta$ & $r$    & Learning rate      & Weight decay   & Hidden units \\ \midrule
			Cora                     & $1.0$     & $0.2$  & $\lfloor d/16 \rfloor$   & $1 \times 10^{-2}$ & $5 \times 10^{-4}$  & $32$           \\
			Citeseer                 & $1.0$     & $0.4$  & $\lfloor d/16 \rfloor$   & $1 \times 10^{-2}$ & $5 \times 10^{-4}$  & $32$           \\
			Pubmed                   & $1.0$     & $0.3$  & $\lfloor d/2048 \rfloor$ & $1 \times 10^{-2}$ & $5 \times 10^{-4}$  & $32$           \\
			ACM                      & $1.0$     & $0.9$  & $\lfloor d/64 \rfloor$   & $1 \times 10^{-2}$ & $5 \times 10^{-4}$  & $32$           \\
			BlogCatalog              & $1.0$     & $0.5$  & $\lfloor d/64 \rfloor$   & $1 \times 10^{-2}$ & $5 \times 10^{-4}$  & $32$           \\
			CoraFull                 & $1.0$     & $0.1$  & $\lfloor d/8 \rfloor$    & $1 \times 10^{-2}$ & $5 \times 10^{-4}$  & $32$           \\
			Flickr                   & $1.0$     & $1.0$  & $\lfloor d/64 \rfloor$   & $1 \times 10^{-2}$ & $5 \times 10^{-4}$  & $32$           \\
			UAI                      & $1.0$     & $1.0$  & $\lfloor d/16 \rfloor$   & $1 \times 10^{-2}$ & $5 \times 10^{-4}$  & $32$           \\ \bottomrule[1pt]
		\end{tabular}\label{ParaSettings}}
\end{table*}

	Absorbing the scale $\frac{\alpha\beta}{\alpha\beta-\alpha+1}$ into the to-be-learnt variable $\mathbf{\Theta}^{(l)}$, and substituting $\lambda = \frac{-\alpha\beta+2\alpha-1}{\alpha\beta-\alpha+1}$ and $\mu = \frac{1}{\beta}-1$ into Eq.~\eqref{LFClosd}, we can harvest
	\begin{equation}
	\small
	\begin{aligned}
	& \frac{\alpha\beta}{\alpha\beta-\alpha+1}(\mathbf{I} + \lambda\widehat{\mathbf{A}})^{-1}(\mathbf{I} + \mu\widehat{\mathbf{A}}) \\
	& = \frac{\alpha\beta}{\alpha\beta-\alpha+1}(\mathbf{I} + \frac{-\alpha\beta+2\alpha-1}{\alpha\beta-\alpha+1}\widehat{\mathbf{A}})^{-1}(\mathbf{I} + (\frac{1}{\beta}-1)\widehat{\mathbf{A}}) \\
	& = \left((1-\frac{1}{\beta}+\frac{1}{\alpha\beta})\mathbf{I} + (-1 + \frac{2}{\beta} - \frac{1}{\alpha\beta})\widehat{\mathbf{A}}\right)^{-1}(\mathbf{I} + (\frac{1}{\beta}-1)\widehat{\mathbf{A}})\\
	& = \left((\beta-1+\frac{1}{\alpha})\mathbf{I} + (-\beta + 2 - \frac{1}{\alpha})\widehat{\mathbf{A}}\right)^{-1}(\beta\mathbf{I} + (1-\beta)\widehat{\mathbf{A}}). 
	\end{aligned}
	\end{equation}
	
	By projecting $\mathbf{H}^{*}$ onto $\mathcal{S}^{(l)}$, we can attain 
	\begin{align}
	\mathbf{H}^{(l)} = \sigma\left(\mathbf{H}^{*}\right), l \in [L],      
	\end{align}
	
	For notation consistency, we denote $\lambda$ and $\mu$ as $\alpha$ and $\beta$ in Table \ref{ConnectionRegularizerGCN}, completing the proof. 
	
	The above analyses reveal that when the regularizer is set to $\frac{1}{2}{\rm{Tr}}\left(\mathbf{H}^{\top}(\mathbf{I} + \mu\widehat{\mathbf{A}})^{-1}(\mathbf{I} + \lambda\widehat{\mathbf{A}})\mathbf{H}\right)$, the framework \eqref{regularizer_framework} can generate GNN-LF \cite{ZhuWang2021Interpreting}.
\end{proof}

\subsection{More Experimental Settings and Results}
In this part, we provide more experimental settings and results for tsGCN, including hyperparameter settings, t-SNE visualizations of various methods, and the parameter sensitivity analysis of tsGCN w.r.t. F1-score.

\textbf{Hyperparameter Settings.} The detailed values of several hyperpaerameters are recorded in Table \ref{ParaSettings}, which can be used to reproduce the reported experimental results. And the code is also provided as a supplementary file.

\textbf{More visualizations.} We draw the t-SNE of embeddings generated by all methods on all datasets from Fig.~\ref{Tsne_Cora} to Fig.~\ref{Tsne_UAI}, from which we have the following observations:
\begin{itemize}
    \item The results are generally matched with the quantitative performance, i.e., tsGCN achieves better results than the others on most datasets.
    \item Embeddings generated by tsGCN achieve better inter-class separation alongside intra-class clustering than those generated by tsGCN (inv), even when their quantitative performance is comparable.
\end{itemize}
\textbf{Parameter Sensitivity.} It can be seen that the F1-scores of tsGCN w.r.t. ($\alpha$, $\beta$) hold the similar trends with the classification accuracies of tsGCN.

\begin{figure*}[!tbp]
	\centering
	\includegraphics[width=0.99\textwidth]{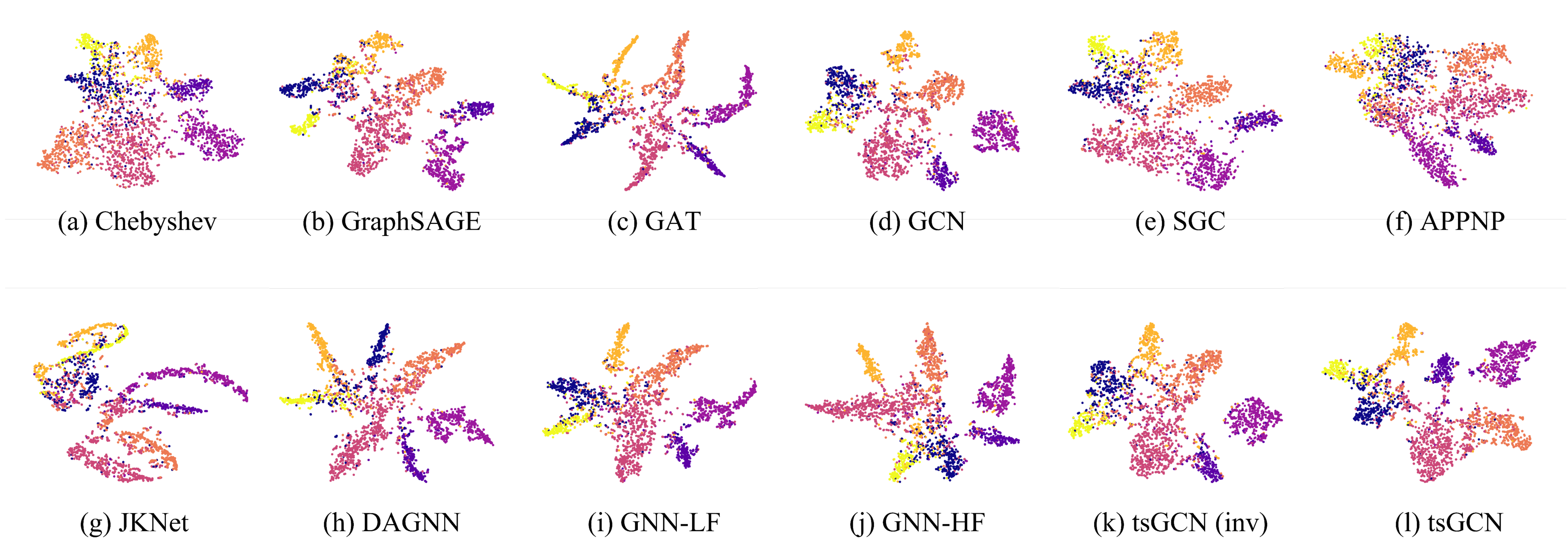}\\
	\caption{Different methods' t-SNE visualizations on Cora, where each color corresponds to one class.}
	\label{Tsne_Cora}
\end{figure*}

\begin{figure*}[!tbp]
	\centering
	\includegraphics[width=0.99\textwidth]{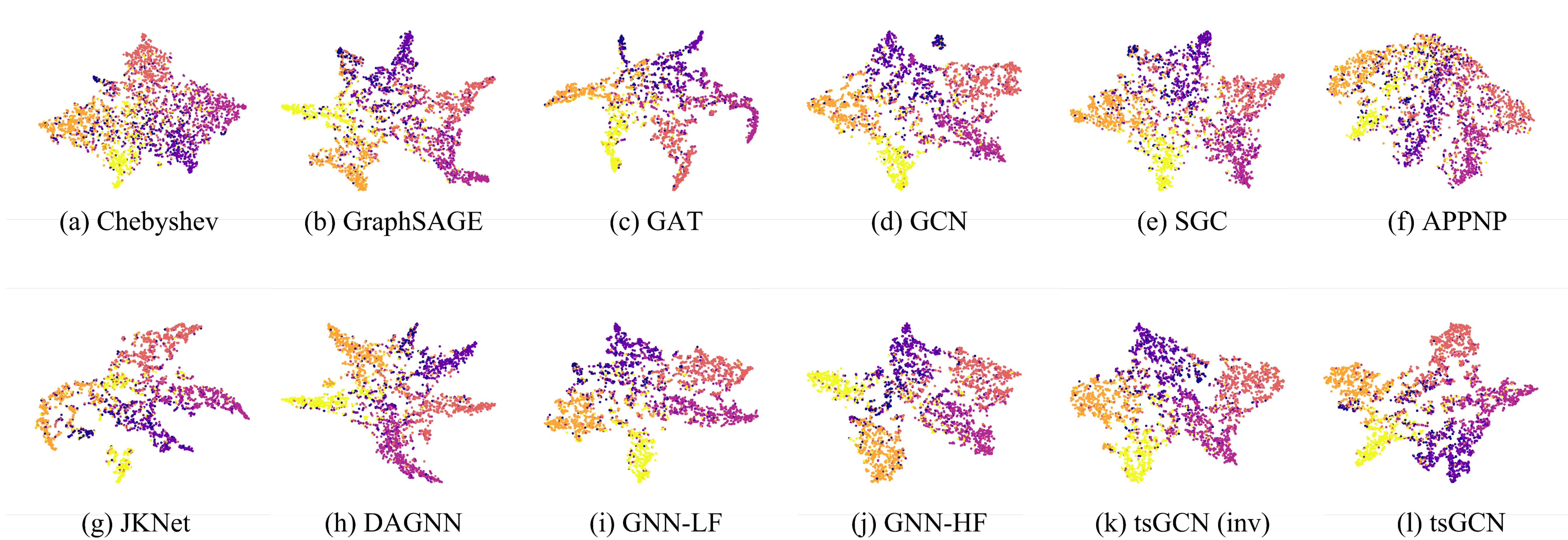}\\
	\caption{Different methods' t-SNE visualizations on Citeseer, where each color corresponds to one class.}
	\label{Tsne_Citeseer}
\end{figure*}

\begin{figure*}[!tbp]
	\centering
	\includegraphics[width=0.99\textwidth]{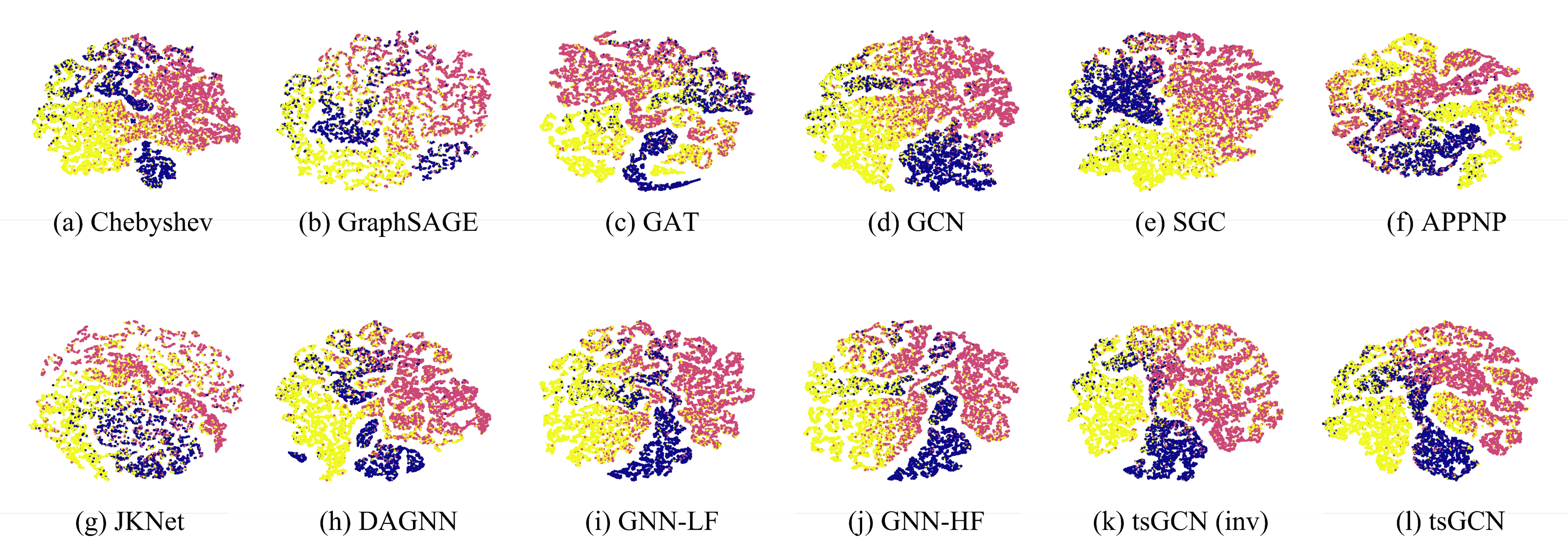}\\
	\caption{Different methods' t-SNE visualizations on Pubmed, where each color corresponds to one class.}
	\label{Tsne_Pubmed_}
\end{figure*}

\begin{figure*}[!tbp]
	\centering
	\includegraphics[width=0.99\textwidth]{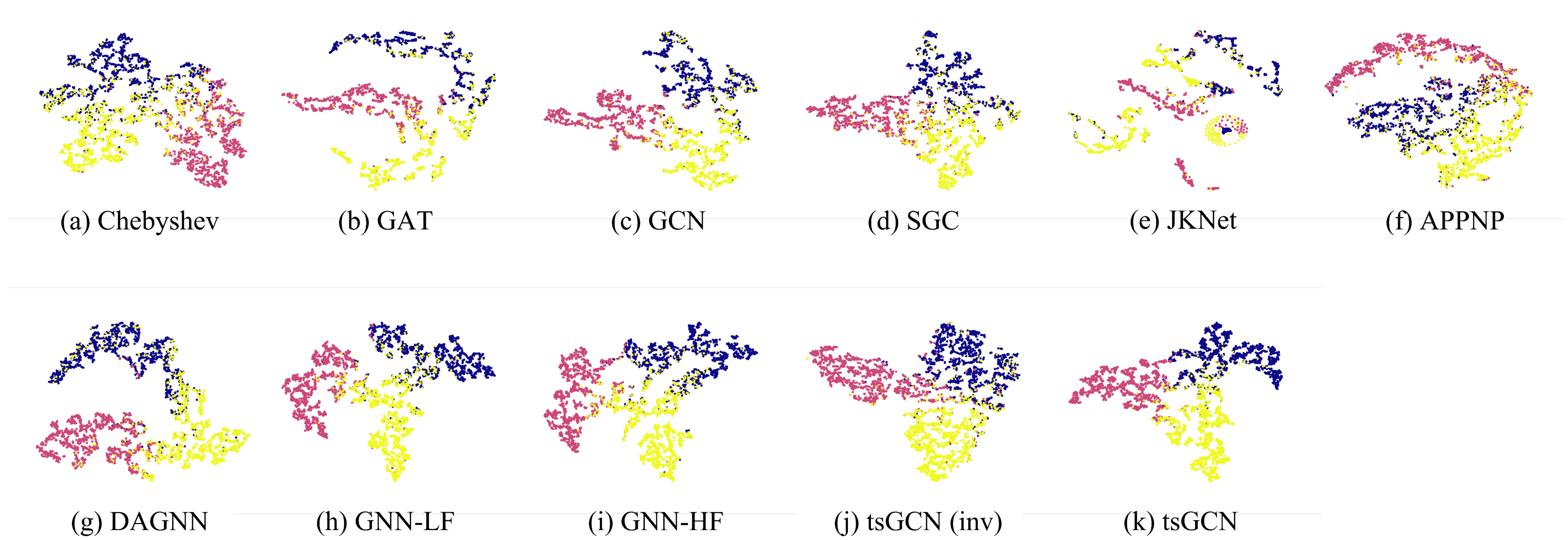}\\
	\caption{Different methods' t-SNE visualizations on ACM, where each color corresponds to one class. Note that GraphSAGE fails to run on ACM.}
	\label{Tsne_ACM}
\end{figure*}

\begin{figure*}[!tbp]
	\centering
	\includegraphics[width=0.99\textwidth]{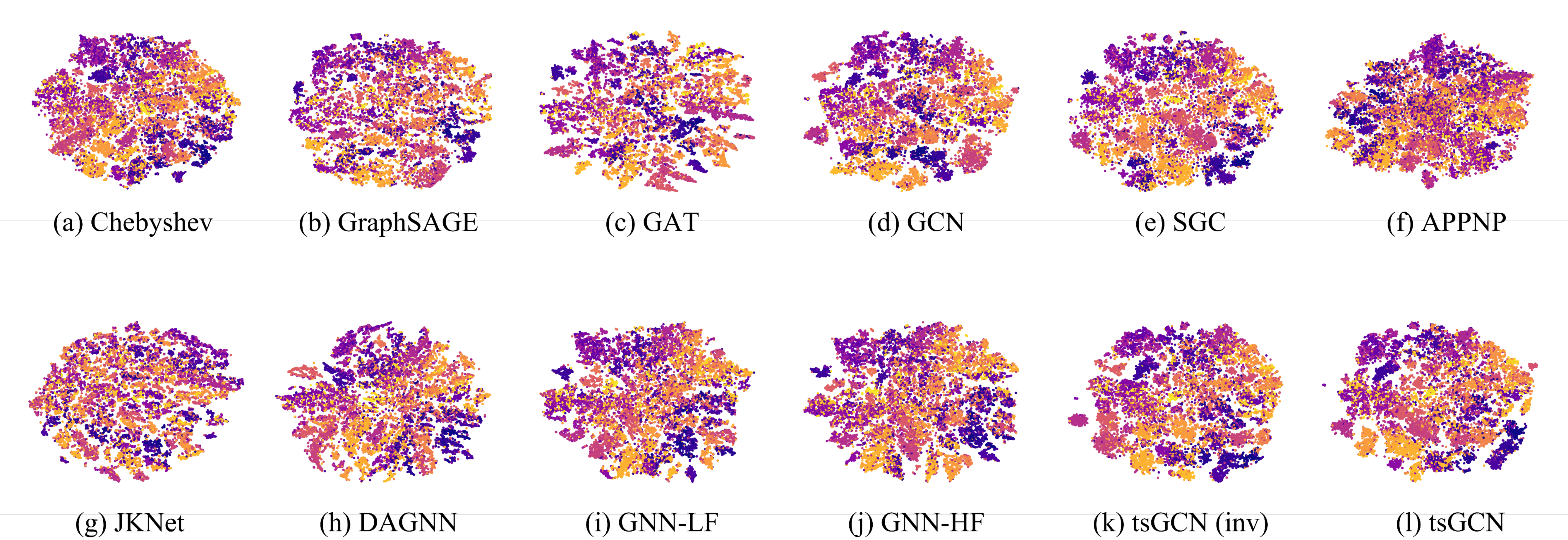}\\
	\caption{Different methods' t-SNE visualizations on CoraFull, where each color corresponds to one class.}
	\label{Tsne_CoraFull_}
\end{figure*}

\begin{figure*}[!tbp]
	\centering
	\includegraphics[width=0.99\textwidth]{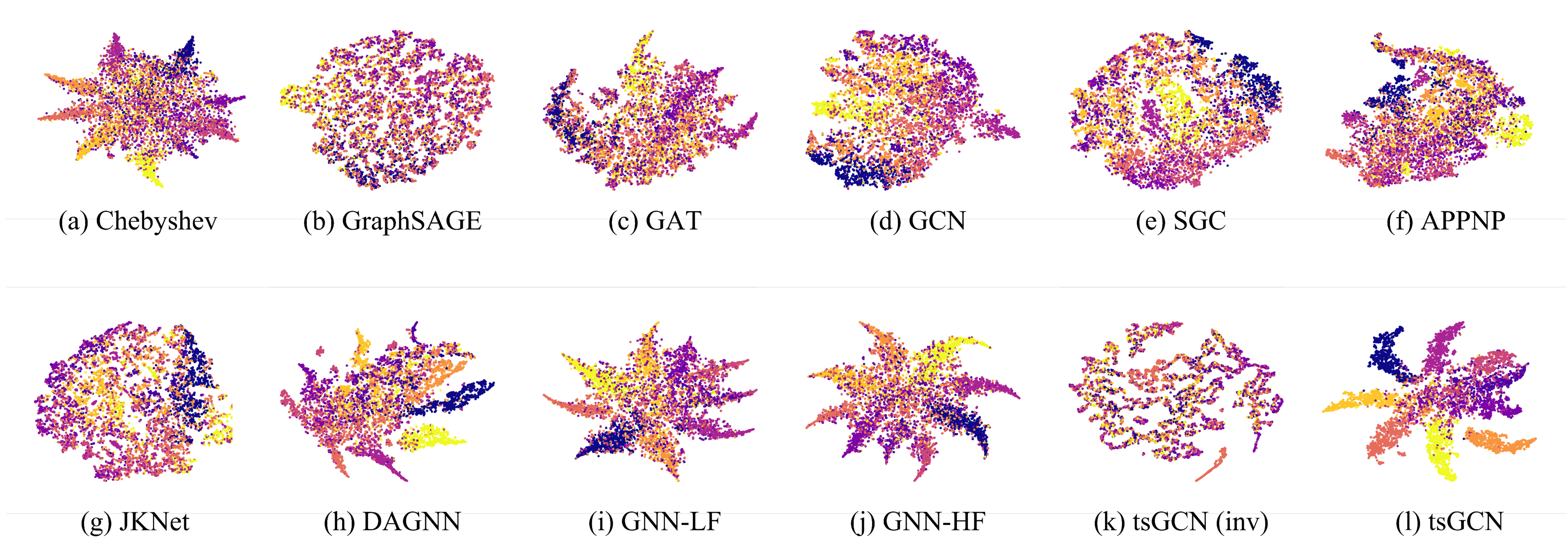}\\
	\caption{Different methods' t-SNE visualizations on Flickr, where each color corresponds to one class.}
	\label{Tsne_Flickr}
\end{figure*}

\begin{figure*}[!tbp]
	\centering
	\includegraphics[width=0.99\textwidth]{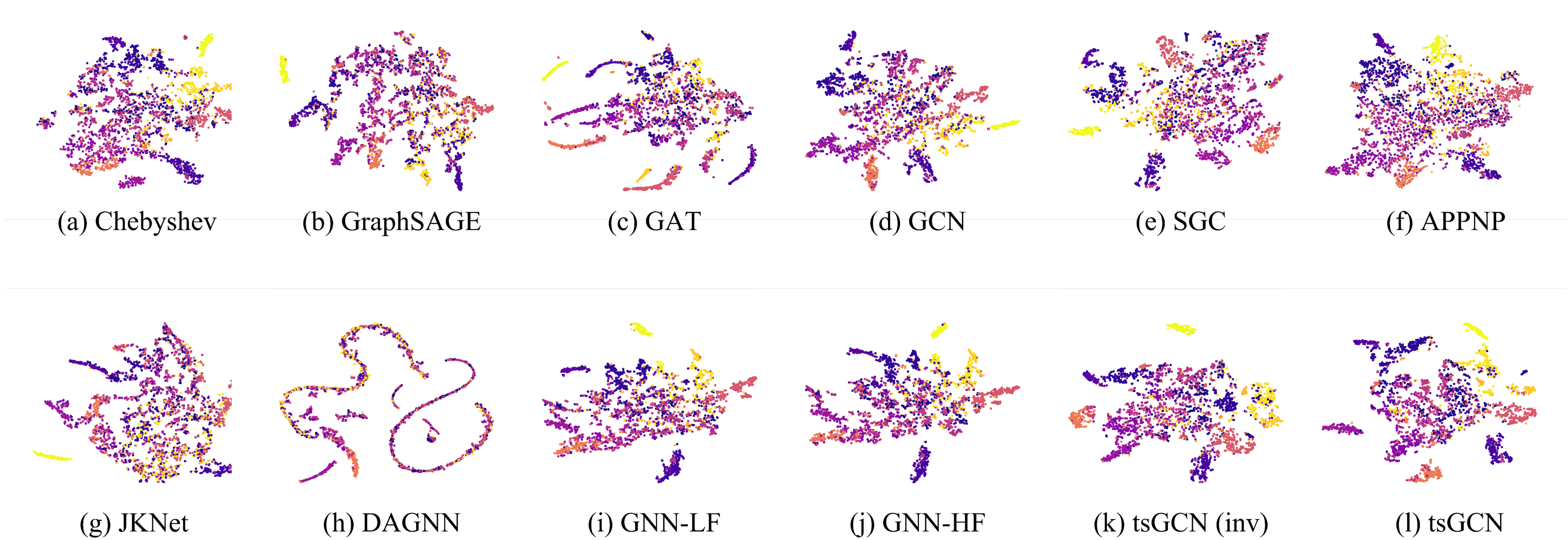}\\
	\caption{Different methods' t-SNE visualizations on UAI, where each color corresponds to one class.}
	\label{Tsne_UAI}
\end{figure*}

\begin{figure*}[!tbp]
	\centering
	\includegraphics[width=0.99\textwidth]{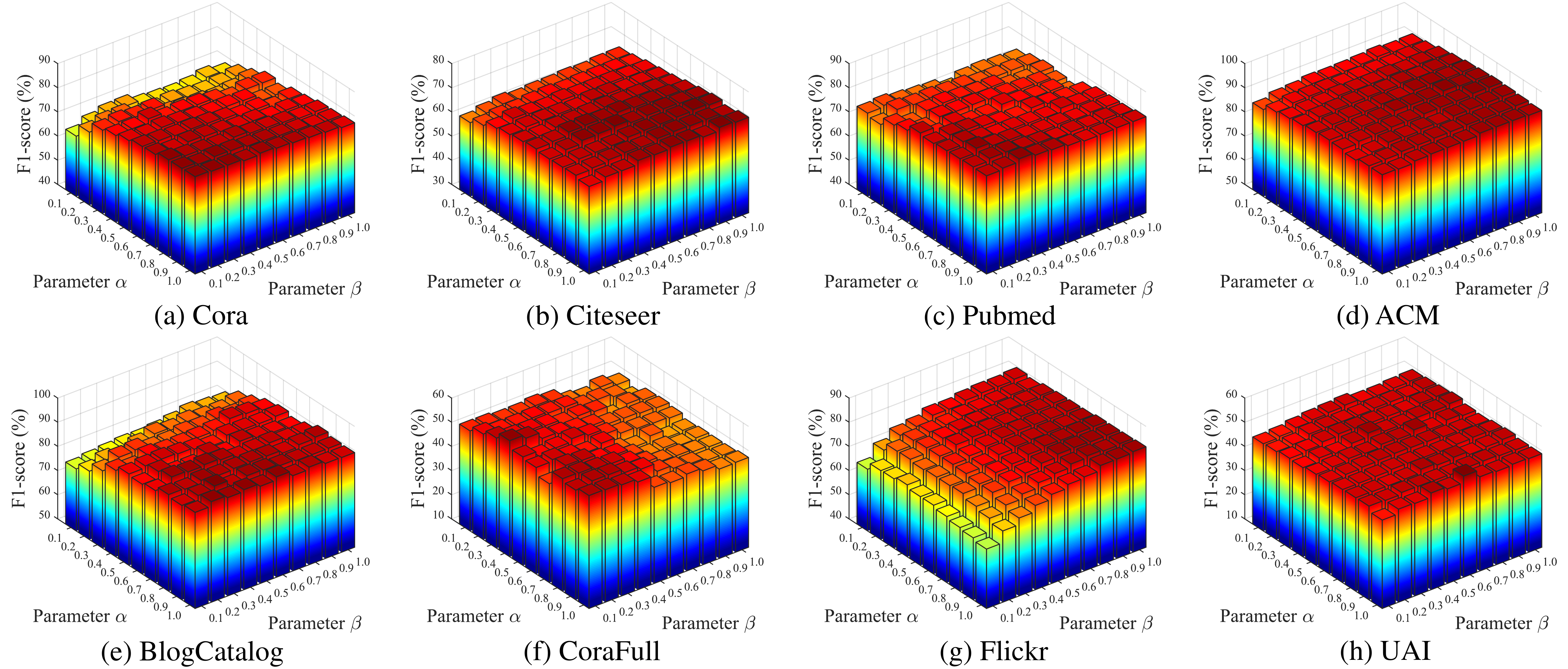}\\
	\caption{The classification F1-scores of tsGCN w.r.t. different hyperparameters $\alpha$ and $\beta$ on all datasets.}
	\label{Sensitivity_f1}
\end{figure*}

\bibliographystyle{num}

\end{document}